\documentclass{article}

\PassOptionsToPackage{numbers, compress}{natbib}
\usepackage[preprint]{neurips_2025}

\usepackage[utf8]{inputenc} 
\usepackage[T1]{fontenc}    
\usepackage{hyperref}  
\usepackage{bbm}
\usepackage{url}            
\usepackage{booktabs}       
\usepackage{amsfonts}       
\usepackage{nicefrac}       
\usepackage{microtype}      
\usepackage[dvipsnames]{xcolor}         
\usepackage{multirow}

\usepackage{hhline}

\usepackage{graphicx}
\usepackage{subcaption}
\usepackage[shortlabels]{enumitem}
\usepackage{amsmath}
\usepackage{amssymb}
\usepackage{mathtools}
\usepackage{amsthm}
\usepackage{mathrsfs} 
\usepackage{pifont}
\usepackage{tikz}

\usetikzlibrary{arrows.meta}

\newcommand{\NNW}{W}
\newcommand{\NNB}{b}
\newcommand{\BSBS}{\psi}
\newcommand{\Besov}{\mathcal{B}}
\newcommand{\dnaught}{d_\circ}
\newcommand{\tnaught}{t_\circ}
\newcommand{\bfsnaught}{\bfs_\circ}

\definecolor{darkgreen}{rgb}{0.0, 0.6, 0.0}

\newcommand{\redx}{{\color{red}\ding{55}}}
\newcommand{\greench}{{\color{darkgreen}\ding{51}}}

\makeatletter
\let\orgdescriptionlabel\descriptionlabel
\renewcommand*{\descriptionlabel}[1]{%
  \let\orglabel\label
  \let\label\@gobble
  \phantomsection
  \edef\@currentlabel{#1\unskip}%
  \let\label\orglabel
  \orgdescriptionlabel{#1}%
}
\makeatother

\setlist[description]{%
  labelwidth=!,       
  labelsep=0.5em,     
  leftmargin=2.5em,       
  align=left,
  style=multiline
}

\makeatletter
\newcommand{\labitem}[2]{%
\def\@itemlabel{\textbf{#1}}
\item
\def\@currentlabel{#1}\label{#2}}
\makeatother

\usepackage[capitalize,noabbrev]{cleveref}

\theoremstyle{plain}
\newtheorem{theorem}{Theorem}[section]
\newtheorem*{theorem*}{Theorem}

\newtheorem{lemma}[theorem]{Lemma}
\newtheorem*{lemma*}{Lemma}
\newtheorem{corollary}[theorem]{Corollary}
\theoremstyle{definition}

\newtheorem{example}[theorem]{Example}
\newtheorem{remark}[theorem]{Remark}

\newcommand{\bfs}{s}

\newcommand{\bc}{\begin{center}}
\newcommand{\ec}{\end{center}}
\newcommand{\be}{\begin{equation}}
\newcommand{\ee}{\end{equation}}
\newcommand{\been}{\begin{equation*}}
\newcommand{\eeen}{\end{equation*}}
\newcommand{\ba}{\begin{array}}
\newcommand{\ea}{\end{array}}
\newcommand{\bean}{\setlength\arraycolsep{2pt}\begin{eqnarray*}}
\newcommand{\eean}{\end{eqnarray*}}
\newcommand{\bea}{\setlength\arraycolsep{2pt}\begin{eqnarray}}
\newcommand{\eea}{\end{eqnarray}}
\newcommand{\ben}{\begin{enumerate}}
\newcommand{\een}{\end{enumerate}}
\newcommand{\bed}{\begin{itemize}}
\newcommand{\eed}{\end{itemize}}

\newcommand{\psingle}{P}

\newcommand{\UB}{\mathcal{U}\mathcal{B}}

\DeclarePairedDelimiterX\Set[1]\{\}{%
  
  #1
}
\DeclarePairedDelimiterXPP\prob[1]{\mathbb{P}}(){}{
  
  #1
}
\DeclarePairedDelimiterXPP\expt[1]{\mathbb{E}}[]{}{
  
  #1
}

\DeclarePairedDelimiterXPP\exptd[2]{\mathbb{E}^{#2}}[]{}{
  
  #1
}

\DeclarePairedDelimiterXPP\Log[1]{\operatorname{log}}(){}{#1}
\DeclarePairedDelimiterXPP\Exp[1]{\operatorname{exp}}(){}{#1}
\DeclarePairedDelimiterXPP\Covering[1]{\mathcal{N}}(){}{#1}
\DeclarePairedDelimiterXPP\MEPR[1]{\operatorname{MEPR}}(){}{#1}
\DeclarePairedDelimiterXPP\Gaussian[2]{\mathrm{N}}(){}{#1,#2}
\DeclarePairedDelimiterXPP\Ber[1]{\mathrm{Ber}}(){}{#1}
\DeclarePairedDelimiter\abs{\lvert}{\rvert}%
\DeclarePairedDelimiter\norm{\lVert}{\rVert}%
\DeclarePairedDelimiter\pbrac{(}{)}%
\DeclarePairedDelimiter\sbrac{[}{]}%
\DeclarePairedDelimiter\cbrac{\{}{\}}%

\makeatletter
\let\oldSet\Set
\def\Set{\@ifstar{\oldSet}{\oldSet*}}
\let\oldprob\prob
\def\prob{\@ifstar{\oldprob}{\oldprob*}}
\let\oldexpt\expt
\def\expt{\@ifstar{\oldexpt}{\oldexpt*}}
\let\oldexptd\exptd
\def\exptd{\@ifstar{\oldexptd}{\oldexptd*}}
\let\oldLog\Log
\def\Log{\@ifstar{\oldLog}{\oldLog*}}
\let\oldExp\Exp
\def\Exp{\@ifstar{\oldExp}{\oldExp*}}
\let\oldMEPR\MEPR
\def\MEPR{\@ifstar{\oldMEPR}{\oldMEPR*}}
\let\oldabs\abs
\def\abs{\@ifstar{\oldabs}{\oldabs*}}
\let\oldnorm\norm
\def\norm{\@ifstar{\oldnorm}{\oldnorm*}}
\let\oldBer\Ber
\def\Ber{\@ifstar{\oldBer}{\oldBer*}}
\let\oldGaussian\Gaussian
\def\Gaussian{\@ifstar{\oldGaussian}{\oldGaussian*}}
\let\oldCovering\Covering
\def\Covering{\@ifstar{\oldCovering}{\oldCovering*}}
\let\oldpbrac\pbrac
\def\pbrac{\@ifstar{\oldpbrac}{\oldpbrac*}}
\let\oldsbrac\sbrac
\def\sbrac{\@ifstar{\oldsbrac}{\oldsbrac*}}
\let\oldcbrac\cbrac
\def\cbrac{\@ifstar{\oldcbrac}{\oldcbrac*}}
\makeatother
\newcommand{\dprime}{{\prime\prime}}

\DeclareMathOperator*{\argmin}{arg\,min}

\def\lemmaautorefname~#1\null{Lemma #1\null}
\def\theoremautorefname~#1\null{Theorem #1\null}
\def\sectionautorefname~#1\null{Section #1\null}
\def\exampleautorefname~#1\null{Example #1\null}
\def\subsectionautorefname~#1\null{Section #1\null}

\newcommand{\aref}[1]{{\rm\ref{#1}}}

\title{Posterior Contraction for Sparse Neural Networks\\ in Besov Spaces with Intrinsic Dimensionality}

\author{
 Kyeongwon Lee\textsuperscript{1} \quad Lizhen Lin\textsuperscript{1} \quad  Jaewoo Park\textsuperscript{2} \quad Seonghyun Jeong\textsuperscript{2}\thanks{Corresponding author}
 \\
  \textsuperscript{1}Department of Mathematics, University of Maryland \\
  \textsuperscript{2}Department of Statistics and Data Science, Yonsei University\\
  \texttt{\{kwlee,lizhen01\}@umd.edu \quad \{jwpark88,sjeong\}@yonsei.ac.kr}
}

\begin{document}

\maketitle

\begin{abstract}
This work establishes that sparse Bayesian neural networks achieve optimal posterior contraction rates over anisotropic Besov spaces and their hierarchical compositions. These structures reflect the intrinsic dimensionality of the underlying function, thereby mitigating the curse of dimensionality.
Our analysis shows that Bayesian neural networks equipped with either sparse or continuous shrinkage priors attain the optimal rates which are dependent on the intrinsic dimension of the true structures. Moreover, we show that these priors enable \emph{rate adaptation}, allowing the posterior to contract at the optimal rate even when the smoothness level of the true function is unknown. The proposed framework accommodates a broad class of functions, including additive and multiplicative Besov functions as special cases.
These results advance the theoretical foundations of Bayesian neural networks and provide rigorous justification for their practical effectiveness in high-dimensional, structured estimation problems.
\end{abstract}

\section{Introduction}

Neural networks (NNs) have been widely used to extract features from complex datasets, such as visual recognition and language modeling \citep{goodfellow2016deep}. Due to their approximation ability \citep{bianchini2014complexity,cybenko1989approximation, delalleau2011shallow,hornik1991approximation,telgarsky2015representation,telgarsky2016benefits}, NNs exhibit remarkable flexibility in representing complex multivariate functions. Given their network structure, NNs are trained by minimizing empirical risk defined through a suitable loss function, which can be viewed as maximum likelihood estimation from a statistical perspective \citep{schmidt2020nonparametric}.  However, such a frequentist approach may lead to miscalibration and overconfidence \citep{minderer2021revisiting}, particularly when the model is trained on out-of-distribution samples \citep{hein2019relu}. On the other hand, the Bayesian approach quantifies predictive uncertainty via the posterior distribution \citep{blundell2015weight,gal2016dropout} and can improve calibration and robustness in many cases \citep{arbel2023primer,daxberger2021bayesian}. 
Another key advantage of the Bayesian perspective is that, unlike their frequentist counterparts, Bayesian neural networks (BNNs) can easily achieve the optimal rate of convergence without knowledge of the smoothness of the function. This property, known as \emph{rate adaptation}, is inherently attained by the nature of Bayesian inference.

To capture the flexibility of NNs, it is desirable to consider function classes that are richer than the intuitively smooth ones, such as H\"older or Sobolev spaces.
In particular, Besov spaces, which encompass non-smooth or even discontinuous functions, are well suited for this purpose. For example, images often exhibit local inhomogeneities and sharp edges naturally represented within Besov spaces \citep{donoho1995adapting}. 
Moreover, additional flexibility arises from accounting for intrinsic dimensionality, which helps explain the robustness of NNs to the curse of dimensionality. 
For instance, images frequently possess low intrinsic dimensionality \citep{pope2021intrinsic}.
Such an intrinsic dimensional structure can be implicitly modeled through anisotropy or composite function constructions. Although frequentist approaches have extensively investigated theoretical properties in these complex settings \citep{suzuki2018adaptivity,suzuki_deep_2021}, Bayesian analyses remain largely limited to simpler cases \citep{kong2024posterior,lee2022asymptotic,polson2018posterior,castillo2024posterior}.
This study aims to fill this gap by establishing that BNNs achieve optimal posterior contraction and exhibit rate adaptation under the fully Bayesian paradigm in such complex scenarios.


\subsection{Related works and our contribution}

Rates of convergence provide a fundamental means of assessing the quality of estimation procedures in both frequentist and Bayesian inference. While frequentist analyses focus on the convergence rates of estimators, Bayesian approaches examine the contraction rates of posterior distributions, which describe how quickly the posterior concentrates around the truth. In the deep neural network (DNN) literature, such notions of optimality have been investigated under a variety of settings from both frequentist and Bayesian perspectives.

\textbf{Frequentist works.}
Under the frequentist paradigm, \citet{schmidt2020nonparametric} showed that a carefully constructed DNN estimator achieves the near-minimax optimal rate over H\"older classes, provided that the network architecture has sufficient depth and sparsity adapted to the underlying function complexity. Their result extends to composite structures, encompassing a wide range of structured function classes, including additive models. \citet{suzuki2018adaptivity} further extended these results to Besov spaces, and \citet{suzuki_deep_2021} demonstrated that DNNs mitigate the curse of dimensionality by adapting to anisotropic smoothness and composite structures in Besov spaces. In contrast to the aforementioned works that utilize sparse networks, \citet{kohler2021rate} investigated dense NNs for composite H\"older spaces. Collectively, these frequentist results demonstrate that DNNs can achieve optimal convergence rates under appropriately specified conditions. However, attaining such optimality requires careful tuning of the network architecture to the underlying function complexity, and rate-adaptive results are currently not available.

\textbf{Bayesian works.}
The Bayesian literature remains comparatively limited. A pioneering contribution by \citet{polson2018posterior} established that BNNs equipped with spike-and-slab priors achieve optimal posterior contraction over H\"older classes. This work was extended to Besov spaces by \citet{lee2022asymptotic}, who considered both spike-and-slab and continuous shrinkage priors, although the application of shrinkage priors remains largely limited. \citet{kong2024posterior} obtained similar results using dense networks with non-sparse priors, further extending the theory to composite H\"older spaces. To the best of our knowledge, the only Bayesian study that considers anisotropic Besov spaces is \citet{castillo2024posterior}. However, that work primarily employs the fractional posterior approach \citep{bhattacharya2019Bayesian,yang2020alpha}, and the standard posterior under the fully Bayesian framework has been investigated only under strong restrictions. Moreover, they focused on H\"older classes in the composite setting, and, to our knowledge, no Bayesian study has yet considered composite structures within anisotropic Besov spaces.
Despite this limited scope, these Bayesian approaches achieve rate adaptation, attaining optimal posterior contraction without prior knowledge of the smoothness level.


\textbf{Our contribution.}
We establish that sparse BNNs achieve the near-minimax optimal rates over anisotropic and composite Besov spaces. This result demonstrates that BNNs adapt to intrinsic dimensional structures, thereby avoiding the curse of dimensionality. We show that sparse BNNs can accommodate a broader class of realistic functions that have not been fully addressed in earlier works. Compared to \citet{castillo2024posterior}, we adopt the pure Bayesian framework, relying on the standard posterior distribution for inference. This choice aligns more closely with conventional Bayesian practice and facilitates broader acceptance within the Bayesian community. Furthermore, we show that sparse BNNs adapt to the underlying model complexity and attain the optimal rate for the target function without requiring oracle knowledge. As a result, users are not required to specify the exact network architecture; rate adaptation is achieved through an appropriately designed prior distribution. A summary of related works and our contributions is provided in \Cref{tbl:related_works_modified}.

\begin{table}[t!]
 \caption{
 Summary of related works and this study. The abbreviations Iso., Aniso.,  and Ada. denote \textit{isotropic}, \textit{anisotropic}, and \textit{adaptation}, respectively. }
    \centering
   \resizebox{\textwidth}{!}{  
    \begin{tabular}{lllllll}
    \toprule
    & & & \multicolumn{2}{c}{Single Function} &  \multicolumn{2}{c}{Composite Function} \\    
\cmidrule(lr){4-5} \cmidrule(lr){6-7}
    Study & Approach & Architecture & Function Space & Ada. & Function Space & Ada. \\
    \midrule
    \citet{schmidt2020nonparametric} & Frequentist & Sparse  & Iso. H\"older & \redx & Iso. H\"older  & \redx \\
    \citet{kohler2021rate} & Frequentist & Dense & Iso. H\"older & \redx & Iso. H\"older  & \redx \\
    \citet{suzuki2018adaptivity} & Frequentist & Sparse & Iso. Besov & \redx & -  & - \\
    \citet{suzuki_deep_2021} & Frequentist & Sparse & Aniso. Besov & \redx &  Aniso. Besov  & \redx \\
    \citet{polson2018posterior} & Bayesian & Sparse & Iso. H\"older & \greench & - & - \\
   \citet{kong2024posterior} & Bayesian & Dense & Iso. H\"older & \greench & Iso. H\"older & \greench \\
    \citet{lee2022asymptotic} & Bayesian & Sparse & Iso. Besov & \greench & - & - \\
    \citet{castillo2024posterior} & Bayesian$^\ast$ & Dense & Aniso. Besov$^{\ast\ast}$ & \greench & Iso. H\"older & \greench  \\
    This work & Bayesian & Sparse & Aniso. Besov & \greench & Aniso. Besov & \greench \\
    \bottomrule
    \end{tabular}
    }
        \begin{minipage}{0.97\linewidth}
    \scriptsize $\ast$  \citet{castillo2024posterior} primarily employ the fractional posterior to circumvent issues of model complexity.
    
    \scriptsize $\ast\ast$  \citet{castillo2024posterior} impose stronger restrictions on the smoothness parameter than other studies on anisotropic Besov spaces (see \Cref{sec:ratebesov} for details).
    
    \end{minipage}
    \label{tbl:related_works_modified}
\end{table}

\section{Preliminaries}\label{sec:prelim}

\subsection{Setup}

\textbf{Notation.}
For $a \in \mathbb{R}$, let $\lfloor a \rfloor$ and $\lceil a \rceil$ denote the floor and ceiling functions, respectively. For $n \in \mathbb{N}$, the notation $[n]$ stands for the set $\Set{1, 2, \cdots, n}$. For $a, b \in \mathbb{R}$, we write $a \vee b$ and $a\wedge b$ to denote $\max\Set{a, b}$ and $\min\Set{a, b}$, respectively.
For a real vector $v$, $\norm{v}_p$ denotes the $\ell_p$-norm for $p\in[1,\infty]$, and $\norm{v}_0$ denotes the number of nonzero components.
For a measurable function $f:[0,1]^d \rightarrow \mathbb R$ and a measure $\mu$, define $\norm{f}_{L^p(\mu)} = (\int_{[0,1]^d} |f|^p d\mu)^{1/p}$ for $p\in (0,\infty)$.
When $\mu$ is the Lebesgue measure, we write $\norm{f}_{L^p}$ for brevity.
For $p=\infty$, we define $\norm{f}_{L^\infty} = \operatorname*{ess\,sup}_{x\in[0,1]^d} |f(x)|$. We also define the supremum norm by $\lVert f \rVert_\infty = \sup_{x\in[0,1]^d} |f(x)|$.
The Dirac delta at zero is denoted by $\delta_0$, and the indicator function of a set $A$ is denoted by $I(A)$. For sequences $a_n$ and $b_n$, we write $a_n \lesssim b_n$ and $a_n \gtrsim b_n$ to mean that $a_n \leq C b_n$ for some universal constant $C > 0$. 
If $a_n\lesssim b_n\lesssim a_n$, we write $a_n\asymp b_n$.
Let $\UB= \{f:[0,1]^d\rightarrow \mathbb R; \lVert f\rVert_{\infty} \leq 1\}$
denote the family of uniformly bounded functions. For a normed space $\mathcal{F}$, $U(\mathcal{F})$ denotes the unit ball of $\mathcal{F}$.

\textbf{Model.}
We consider a nonparametric regression problem with a $d$-dimensional input variable $X_i\in[0,1]^d$ and an output variable $Y_i\in\mathbb R$, for $i\in [n]$.
The observations $\mathcal{D}_n=\{(X_i,Y_i)\}_{i=1}^n$ are independent and identically distributed according to the model,
\begin{equation}
\label{eqn:reg:model}
Y_i = f_0(X_i) + \xi_i,\quad X_i\sim P_X,\quad \xi_i\sim\mathrm{N}(0,\sigma_0^2),\quad i\in [n],
\end{equation}
where $f_0:[0,1]^d\rightarrow \mathbb R$ is the true regression function, $\sigma_0^2>0$ is the noise variance, and $P_X$ is the distribution of $X_i$.  We denote the joint distribution of $\mathcal D_n$ under this model by $\psingle_{f_0,\sigma_0}^{(n)}$.

\textbf{Neural network.}
We denote the ReLU activation function by $\zeta(\cdot)$ and use the same notation for its vectorized version. The parameter space of $L$-layered NNs with $B$-bounded and $S$-sparse weights is defined as
    \begin{align*}
        \Theta(L, D, S, B) = \Big\{ &\theta=\big(\NNW^{(1)}, \NNB^{(1)},\dots, \NNW^{(L+1)},  \NNB^{(L+1)}\big) : \NNW^{(l)} \in \mathbb{R}^{d_{l} \times d_{l-1}}, \ \NNB^{(l)} \in \mathbb{R}^{d_l},
        \\ &d_l=D, \ l \in [L], \ d_0=d, \ d_{L+1}=1, \ \norm{\theta}_0 \leq S, \ \norm{\theta}_{\infty} \leq B
            \Big\}.
    \end{align*}
Let $f_\theta(\cdot) = (\NNW^{(L+1)}(\cdot) + \NNB^{(L+1)}) \circ \zeta \circ \cdots \circ \zeta \circ (\NNW^{(1)}(\cdot) + \NNB^{(1)}):[0,1]^d\rightarrow\mathbb R$ be the input-output mapping of an $L$-layered NN. 
We define the corresponding feedforward NN class as $\Phi(L, D, S, B)=\Set{\mathrm{clip} \circ f_\theta : \theta \in \Theta(L, D, S, B)}$, where $\mathrm{clip}(x) = \min\{1,\max\{-1,x\}\}$ clips its input to the interval $[-1,1]$.
We also define the unbounded sparse and unbounded dense parameter spaces as $\Theta(L, D,S) = \lim_{B \rightarrow \infty} \Theta(L, D, S, B)$ and $\Theta(L, D) = \lim_{S \rightarrow \infty} \Theta(L, D, S)$, respectively. 
By further taking union over the width, we define
$\Theta(L) = \bigcup_{D=1}^\infty\Theta(L, D)$.
The corresponding NN classes are defined analogously as
$\Phi(L, D,S)=\lim_{B \rightarrow \infty} \Phi(L, D, S, B)$, $\Phi(L, D)=\lim_{S \rightarrow \infty} \Phi(L, D,  B)$, and $\Phi(L) = \bigcup_{D=1}^\infty\Phi(L, D)$.

\subsection{Anisotropic Besov spaces}

We first consider anisotropic Besov spaces, which accommodate varying smoothness across different directional components \citep{nikolskii1975approximation,suzuki_deep_2021}.
Let $\bfs = (s_1, \dots, s_d) \in \mathbb{R}_{++}^d$, $0 < p, q \leq \infty$, and $r = \lfloor \max_j s_j \rfloor + 1$. For $h \in \mathbb{R}^d$, the $r$-th order difference is defined as $\Delta_h^r(f)(x) = \sum_{j=0}^r \binom{r}{j} (-1)^{r-j}f(x+jh)$ if $x, x+rh \in [0,1]^d$, and $0$ otherwise.
Given $t=(t_1, \dots, t_d)\in\mathbb{R}_{++}^d$, the anisotropic modulus of smoothness is defined as
$w_{r, p}(f, t) = \sup_{h \in \mathbb{R}^d: |h_j| \leq t_j} \|\Delta_h^r(f)\|_{L^p}$.
The anisotropic Besov seminorm is then defined as $\|f\|_{\Besov_{p,q}^{\bfs}} := \|f\|_{L^p} + \|f\|_{\Besov_{p,q}^{\bfs}}^\ast$, where
\begin{equation*}
\|f\|_{\Besov_{p,q}^{\bfs}}^\ast =
\begin{cases}
\left( \sum_{k=0}^\infty \left[ 2^k w_{r,p} ( f, ( 2^{-k/s_1}, \dots, 2^{-k/s_d} ) ) \right]^q \right)^{1/q} & \text{if } q < \infty, 
\\
\sup_{k \ge 0} \left[ 2^k w_{r,p} ( f, ( 2^{-k/s_1}, \dots, 2^{-k/s_d} ) ) \right] & \text{if } q = \infty.
\end{cases}
\end{equation*}
The anisotropic Besov space $\Besov_{p,q}^{\bfs}$ is the collection of all functions $f \in L^p$ such that $\|f\|_{\Besov_{p,q}^{\bfs}}$ is finite.
For an anisotropic smoothness parameter $\bfs$, we define the smallest smoothness as $\underline{\bfs} = \min_j s_j$, the largest smoothness as $\overline{\bfs} = \max_j s_j$, and the intrinsic smoothness (exponent of global smoothness) as $\tilde{\bfs} = ( \sum_{j=1}^d s_j^{-1} )^{-1}$.
In the special case where $\bfs = (s_0, \dots, s_0)$ for some $s_0 > 0$, the space $\Besov_{p,q}^{\bfs}$ reduces to an isotropic Besov space. Besov spaces generalize classical notions of differentiability and continuity and are more flexible than H\"{o}lder spaces, which can be viewed as particular subspaces of Besov spaces.
For further discussion of the continuous embedding properties, see Remark~\ref{rmk:besov-embeddings}.

It is well known that the minimax optimal rate of convergence over isotropic Besov classes with smoothness $s_0$ is $n^{-{s_0}/{(2s_0+d)}}$ \citep{donoho1998minimax,Gine_Nickl_2016}. For anisotropic Besov spaces, the minimax rate is given by $n^{-\tilde{\bfs}/(2\tilde{\bfs}+1)}$ \citep{hoffman2002random,ibragimov2013statistical}, which can also be expressed as $n^{-\underline{\bfs}/(2\underline{\bfs} + d^\ast)}$, where $d^\ast := \underline{\bfs} / \tilde{\bfs}$. This form resembles the minimax rate over isotropic Besov spaces with smoothness $\underline{\bfs}$ and dimension $d^\ast$, in which the anisotropic space is continuously embedded. In this sense, $d^\ast$ can be interpreted as an intrinsic dimension associated with anisotropic Besov spaces \citep{suzuki_deep_2021}. For example, the two functions $f_1$ and $f_2$ in \Cref{fig:ex_aniso} belong to $\Besov_{1,\infty}^{(s_1, s_2)}$ for any $0 < s_1 < 1$ and $s_2 > s_1$. The anisotropic Besov space $\Besov_{1,\infty}^{(s_1, s_2)}$ is continuously embedded in the isotropic Besov space with $s_1$. By accounting for anisotropy, the rate exponent can be improved from $s_1/(2s_1+2)$ to $\tilde{s}/(2\tilde{s}+1)$, where $\tilde{s} = s_1s_2/(s_1+s_2)$. The intrinsic dimension satisfies $d^\ast = s_1/\tilde{s} = s_1/s_2 + 1 < 2$, indicating a reduced effective dimension compared to the ambient dimension $d=2$.

\begin{figure}[t!]
    \centering
\includegraphics[width=\textwidth]{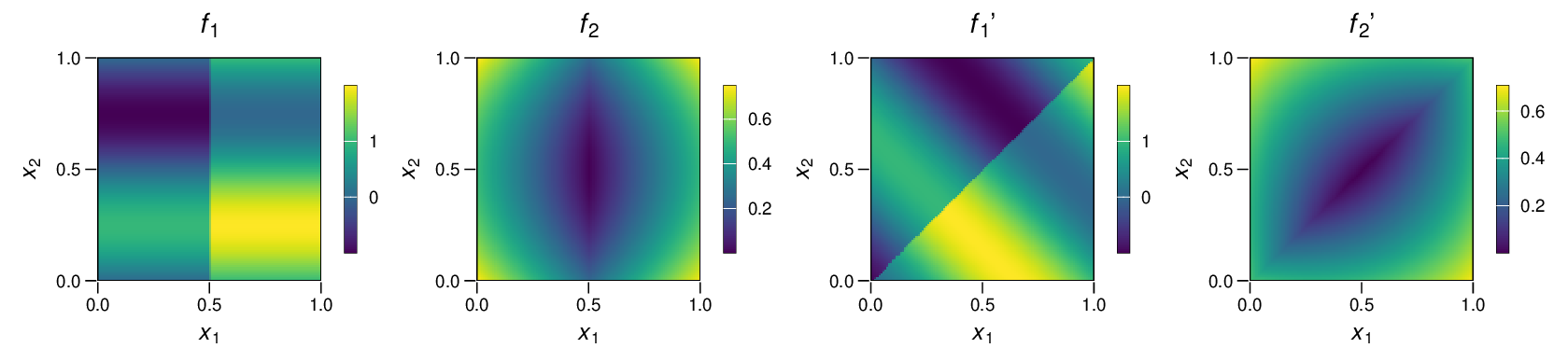}
    \caption{We illustrate two example functions, $f_1(x) = I(\{x_1 \in [1/2, 1]\}) + \sin (2\pi x_2)$ and $f_2(x) = \lvert x_1 - 1/2 \rvert + (x_2 - 1/2)^2$, and their rotated counterparts $f_1'$ and $f_2'$.}\label{fig:ex_aniso}
\end{figure}

\subsection{Composite Besov spaces}
\label{sec:compo}

Beyond anisotropic spaces, we consider composite function spaces in which functions are represented as hierarchical compositions of simpler components. 
For a composite function $f = f_H \circ \cdots \circ f_1 : [0,1]^d\rightarrow \mathbb R$ with composition depth $H$, each intermediate function $f_h$ maps between successive spaces, potentially involving dimension reduction through sparse connectivity. This compositional structure naturally aligns with the layered architecture of NNs, where each layer progressively transforms its inputs into higher-level representations. \Cref{fig:ex_aniso} presents examples in which rotated functions can be interpreted as compositions of Besov functions with affine transformations.
Compositional function spaces offer a useful framework for understanding the ability of NNs to approximate high-dimensional functions without suffering from the curse of dimensionality.

Let $\dnaught = (d^{(0)}, d^{(1)}, \dots, d^{(H)})$ be a sequence of dimensions with $d^{(0)}=d$ and $d^{(H)}=1$. Let $\tnaught = (t^{(1)}, \dots, t^{(H)})$ denote the effective dimensions, where each $t^{(h)}$ satisfies $1 \leq t^{(h)} \leq d^{(h-1)}$. Let $\bfsnaught = (\bfs^{(1)}, \dots, \bfs^{(H)})$ be a sequence of anisotropic smoothness vectors, where each $\bfs^{(h)} \in \mathbb{R}_{++}^{t^{(h)}}$. We define the composite anisotropic Besov space $\Besov_{p,q}^{\dnaught,\tnaught,\bfsnaught}$ as the collection of composite functions $f = f_H \circ \cdots \circ f_1 : [0,1]^d\rightarrow \mathbb R$ such that the intermediate functions $f_h = (f_{h,1}, \dots, f_{h,d^{(h)}}): [0,1]^{d^{(h-1)}} \rightarrow [0,1]^{d^{(h)}}$ for $h=1,\dots, H-1$, and $f_H:[0,1]^{d^{(H-1)}}\rightarrow \mathbb R$ satisfy the following condition: for each $h$ and each $j \in [d^{(h)}]$, there exists a subset $I_{h,j} \subset [d^{(h-1)}]$ with $|I_{h,j}| = t^{(h)}$ and a function $\tilde{f}_{h,j}\in U(\Besov^{\bfs^{(h)}}_{p,q})$ such that $f_{h,j}(x) = \tilde{f}_{h,j}(x_{I_{h,j}})$, where $x_I$ denotes the subvector of $x$ restricted to the coordinates in $I$. 
Additive and multiplicative Besov spaces serve as simple examples of composite Besov spaces, as illustrated in \Cref{fig:composite_structures}.

\begin{figure}[t!]
    \centering
    \begin{subfigure}[b]{0.48\textwidth}
        \centering
        \begin{tikzpicture}[scale=0.95]
            \tikzset{
                input/.style={circle, draw, thick, fill=yellow!30, minimum size=4.5mm, font=\scriptsize, inner sep=0.5pt},
                feature/.style={circle, draw, thick, fill=yellow!30, minimum size=4.5mm, font=\scriptsize, inner sep=0.5pt},
                output/.style={circle, draw, thick, fill=yellow!30, minimum size=4.5mm, font=\scriptsize, inner sep=1pt},
                arrow/.style={-Stealth, thick},
                comp/.style={rectangle, rounded corners, draw, thick, fill=yellow!30, minimum width=4.5mm, minimum height=4.5mm, font=\scriptsize},
            }
            
            \node[input] (x1) at (0,0) {$x_1$};
            \node[input] (x2) at (0,-0.8) {$x_2$};
            \node at (0,-1.3) {\scriptsize$\vdots$};
            \node[input] (xd) at (0,-2) {$x_d$};
            
            \node[feature] (f1) at (1.7,0) {$g_1$};
            \node[feature] (f2) at (1.7,-0.8) {$g_2$};
            \node at (1.7,-1.3) {\scriptsize$\vdots$};
            \node[feature] (fd) at (1.7,-2) {$g_d$};
            
            \draw[dashed, thick, rounded corners, draw=red] (-0.5,0.4) rectangle (2.2,-2.4);
            \node[anchor=south east, font=\scriptsize] at (2.2,0.4) {$t^{(1)}=1\leq d^{(0)}=d$};
            
            \node[comp] (sum) at (3.4,-1) {$\sum$};
            \node[anchor=north west, font=\scriptsize] at (3.2,-1.2) {$t^{(2)}=d^{(1)}$};
            
            \node[output] (out) at (4.7,-1) {$f$};
            
            \draw[arrow] (x1) -- (f1);
            \draw[arrow] (x2) -- (f2);
            \draw[arrow] (xd) -- (fd);
            
            \draw[arrow] (f1) -- (sum);
            \draw[arrow] (f2) -- (sum);
            \draw[arrow] (fd) -- (sum);
            
            \draw[arrow] (sum) -- (out);
            
        \end{tikzpicture}
        \caption{Additive structure}
        \label{fig:additive}
    \end{subfigure}
    \hfill
    \begin{subfigure}[b]{0.48\textwidth}
        \centering
        \begin{tikzpicture}[scale=0.95]
            \tikzset{
                input/.style={circle, draw, thick, fill=yellow!30, minimum size=4.5mm, font=\scriptsize, inner sep=0.5pt},
                feature/.style={circle, draw, thick, fill=yellow!30, minimum size=4.5mm, font=\scriptsize, inner sep=0.5pt},
                output/.style={circle, draw, thick, fill=yellow!30, minimum size=4.5mm, font=\scriptsize, inner sep=1pt},
                arrow/.style={-Stealth, thick},
                comp/.style={rectangle, rounded corners, draw, thick, fill=yellow!30, minimum width=4.5mm, minimum height=4.5mm, font=\scriptsize},
            }
            
            \node[input] (x1) at (0,0) {$x_1$};
            \node[input] (x2) at (0,-0.8) {$x_2$};
            \node at (0,-1.3) {\scriptsize$\vdots$};
            \node[input] (xd) at (0,-2) {$x_d$};
            
            \node[feature] (f1) at (1.7,0) {$g_1$};
            \node[feature] (f2) at (1.7,-0.8) {$g_2$};
            \node at (1.7,-1.3) {\scriptsize$\vdots$};
            \node[feature] (fd) at (1.7,-2) {$g_d$};
            
            \draw[dashed, thick, rounded corners, draw=red] (-0.5,0.4) rectangle (2.2,-2.4);
            \node[anchor=south east, font=\scriptsize] at (2.2,0.4) {$t^{(1)}=1\leq d^{(0)}=d$};
            
            \node[comp] (prod) at (3.4,-1) {$\prod$};
            \node[anchor=north west, font=\scriptsize] at (3.2,-1.2) {$t^{(2)}=d^{(1)}$};
            
            \node[output] (out) at (4.7,-1) {$f$};
            
            \draw[arrow] (x1) -- (f1);
            \draw[arrow] (x2) -- (f2);
            \draw[arrow] (xd) -- (fd);
            
            \draw[arrow] (f1) -- (prod);
            \draw[arrow] (f2) -- (prod);
            \draw[arrow] (fd) -- (prod);
            
            \draw[arrow] (prod) -- (out);
            
        \end{tikzpicture}
        \caption{Multiplicative structure}
        \label{fig:multiplicative}
    \end{subfigure}
    \caption{Illustration of additive ($f(x) = \sum_{i=1}^d g_i(x_i)$) and multiplicative ($f(x) = \prod_{i=1}^d g_i(x_i)$) composite Besov functions. Each component function $g_i$ depends on a single input dimension ($t^{(1)} = 1$), although the ambient dimension $d^{(0)}=d$ may be much larger.}
    \label{fig:composite_structures}
\end{figure}

Note that \citet{suzuki_deep_2021} did not impose effective low-dimensionality in their definition of composite Besov spaces. In contrast, our definition aligns more with the composite Hölder spaces introduced by \citet{schmidt2020nonparametric}, and it is well suited for analyzing functions approximated by NNs, where dimension reduction may occur within layers. By incorporating anisotropy within each $f_h$, the framework captures low-dimensional structures even when the ambient or latent dimensions are high. This interplay between compositional structure and anisotropy enables efficient function representation.

\section{Main results}\label{sec:main_results}

In this section, we present our main results on posterior contraction.
\Cref{table:prior_summary} summarizes the corresponding assumptions and theorems for the anisotropic and composite anisotropic Besov spaces.

\subsection{Posterior contraction in anisotropic Besov spaces}
\label{sec:ratebesov}

We first establish theoretical results for anisotropic Besov spaces. 
The following assumptions are imposed.

\begin{description}
  \item[(A1)\label{assum:a-1}] The true regression function $f_0$ satisfies $f_0\in U(\Besov_{p,q}^{\bfs}) \cap \UB$ for some $0<p, q \leq \infty$ and $\bfs \in \mathbb{R}_{++}^d$, such that $(1/p-1/2)_+ <\tilde{\bfs}$. 
  \item[(A2)\label{assum:a-2}] The distribution $P_X$ has a bounded density $p_X$ such that $\lVert{p_X}\rVert_{\infty} \leq R$ for some constant $R>0$.
  \item[(A3)\label{assum:a-3}] The true standard deviation  $\sigma_0$ satisfies $\underline{\sigma} < \sigma_0 < \overline{\sigma}$ for some constants $\underline{\sigma}$ and $\overline{\sigma}$, and the prior for $\sigma$ is supported on $[\underline{\sigma}, \overline{\sigma}]$ with a positive density throughout.
\end{description}

Assumption~\aref{assum:a-1} states that $f_0$ lies in a bounded Besov class. The condition $(1/p - 1/2)_+ < \tilde{\bfs}$ aligns with those in other theoretical studies on Besov spaces \citep{lee2022asymptotic,suzuki2018adaptivity, suzuki_deep_2021}, whereas \citet{castillo2024posterior} impose a stronger requirement $1/p< \tilde{\bfs}$ to achieve a sharper approximation result.
Assumption~\aref{assum:a-2} is minor given that the dimension $d$ is fixed. Assumption~\aref{assum:a-3} ensures that the prior for $\sigma$ is supported on a compact interval containing $\sigma_0$.
We first show that the optimal rate is achieved by suitably specified spike-and-slab and shrinkage priors, and then demonstrate that rate adaptation can be attained for both priors with slight modifications.

\subsubsection{Spike-and-slab prior}
\label{sec:regression_sas}

\begin{table}[t!]
\caption{Summary of the main results. Along with the listed assumptions, all results also require Assumptions \aref{assum:a-2} and \aref{assum:a-3}, which are taken as common assumptions.
}
\label{table:prior_summary}
\centering
\resizebox{\textwidth}{!}{  
\begin{tabular}{lllll}
\toprule
& \multicolumn{2}{c}{ Anisotropic Besov} & \multicolumn{2}{c}{ Composite Anisotropic Besov} \\
\cmidrule(lr){2-3} \cmidrule(lr){4-5}
{Prior} & {Assumptions} & {Result} & {Assumptions} & {Result} \\
\midrule
Spike-and-Slab &  \aref{assum:a-1}, \aref{cond:ss_support}, \aref{cond:ss_tail} & \cref{thm:reg_ss} & \aref{assum:a-4}/\aref{assum:a-5}, \aref{cond:ss_support}, \aref{cond:ss_tail}  & \cref{thm:general_function}(i) \\
Shrinkage & \aref{assum:a-1}, \aref{cond:shr_support}--\aref{cond:shr_spike} & \cref{thm:shrinkage} &  \aref{assum:a-4}/\aref{assum:a-5}, \aref{cond:shr_support}--\aref{cond:shr_spike}  & \cref{thm:general_function}(ii) \\
Adaptive Spike-and-Slab & \aref{assum:a-1}, \aref{cond:ss_support}, \aref{cond:ss_tail} & \cref{thm:adaptive_estimation}(i) & \aref{assum:a-4}/\aref{assum:a-5}, \aref{cond:ss_support}, \aref{cond:ss_tail} & \cref{thm:general_function}(i) \\
Adaptive Shrinkage & \aref{assum:a-1}, \aref{cond:shr_support}--\aref{cond:shr_spike}  & \cref{thm:adaptive_estimation}(ii) & \aref{assum:a-4}/\aref{assum:a-5}, \aref{cond:shr_support}--\aref{cond:shr_spike} & \cref{thm:general_function}(ii) \\
\bottomrule
\end{tabular}
}
\end{table}

Our results rely on the fact that for every $f_0\in U(\Besov_{p,q}^{\bfs})$, there exists a NN approximator $\hat f\in\Phi(L_{1n},D_{1n},S_{1n},B_1)$ achieving the optimal approximation error, where the network parameters satisfy
\begin{equation}\label{eqn:aniso_optimal_hyperparm}
    \begin{gathered}
        L_{1n}\asymp\log n , \quad D_{1n}\asymp N_n,\quad S_{1n} \asymp N_n \log n ,\quad B_{1} \propto 1,
    \end{gathered}
\end{equation}
with $N_n = \lceil n^{1/(2\tilde{\bfs}+1)} \rceil$. These network parameters depend on the unknown Besov parameters $\tilde s$ and $p$.
See \Cref{lem:anisotropic_approx} and \Cref{rmk:hyper_bigo} for details.

We place a spike-and-slab prior over $\Theta(L_{1n}, D_{1n}, S_{1n})$ using the network parameters in \eqref{eqn:aniso_optimal_hyperparm}. This implies that the network structure is determined based on the smoothness parameter and thus the procedure does not attain rate adaptation.
A spike-and-slab prior for $\theta$ is given by 
\begin{equation}\label{eqn:sparse_prior}
    \begin{aligned}
        \pi(\theta \mid \gamma,L,D,S) & = \prod_{j=1}^{T} \left[ \gamma_j \tilde{\pi}_{SL}(\theta_j)  + (1-\gamma_j) \delta_0(\theta_j) \right] ,\\ 
        \pi( \gamma \mid L,D,S) & = \frac{1}{\binom{T}{S}}I(\gamma \in \Set{0,1}^T,~\norm{\gamma}_0 = S),
    \end{aligned}
\end{equation}
where $\tilde{\pi}_{SL}$ is the slab density for the nonzero components and $T = \abs{\Theta(L,D)}$. The slab distribution is required to satisfy the following assumptions.
\begin{description}
     \item[(B1)\label{cond:ss_support}] 
     $     \log \int_{|u|>K_n} \tilde{\pi}_{SL}(u) d u \lesssim - K_n     $   for any $K_n \rightarrow \infty$.
     \item[(B2)\label{cond:ss_tail}] 
    $\log \inf_{|u| \le B_1}\tilde{\pi}_{SL}(u) \gtrsim -(\log n)^2$. 
\end{description}
Assumption~\ref{cond:ss_support} implies that $\tilde \pi_{SL}$ has exponential tails on both sides. From a nonparametric Bayesian perspective, this condition is necessary to control the prior mass outside a chosen sieve.
Assumption~\ref{cond:ss_tail} requires that $\tilde \pi_{SL}$ to place sufficient mass around the NN approximator, which is essential for ensuring adequate prior concentration in a Kullback-Leibler neighborhood.
If $\tilde \pi_{SL}$ is independent of $n$ and bounded away from zero on $[-B_1, B_1]$, then the assumption holds trivially.

\begin{example}[Uniform slab prior \citep{lee2022asymptotic,polson2018posterior}]\label{ex:uniform}
    If $\tilde{\pi}_{SL}$ is the density of the uniform distribution $U(-C_u, C_u)$ for $C_u>B_1$, then $\tilde{\pi}_{SL}$ satisfies Assumptions~\ref{cond:ss_support}--\ref{cond:ss_tail}.
\end{example}

\begin{example}[Gaussian slab prior]\label{ex:gaussian}
If $\tilde{\pi}_{SL}$ is the density of a zero-mean Gaussian distribution with fixed variance, then $\tilde{\pi}_{SL}$ satisfies Assumptions~\ref{cond:ss_support}--\ref{cond:ss_tail}.
\end{example}

\Cref{ex:uniform} is simple but has the drawback that the prior depends on $B_1$. \Cref{ex:gaussian} mitigates this issue.
We now present the posterior contraction result.

\begin{theorem}[Spike-and-slab prior]\label{thm:reg_ss}
Suppose that Assumptions~\aref{assum:a-1}--\aref{assum:a-3} hold, and that the prior distribution in \eqref{eqn:sparse_prior} is placed over $\Theta(L_{1n}, D_{1n}, S_{1n})$.
Assume further that the slab density $\tilde{\pi}_{SL}$ satisfies Assumptions~\aref{cond:ss_support}--\aref{cond:ss_tail}.
    Then, the posterior distribution concentrates at the rate $\epsilon_n = n^{-\tilde{\bfs}/(2\tilde{\bfs}+1)}(\log n)^{3/2}$, in the sense that
    \begin{equation*}
        \begin{gathered}
        \Pi\Big((f ,\sigma)\in \Phi(L_{1n}, D_{1n}, S_{1n})\times [\underline{\sigma}, \overline{\sigma}]: 
        \lVert{f - f_0}\rVert_{L^2(P_X)} + |\sigma^2 - \sigma_0^2|  > M_n \epsilon_n \mid \mathcal{D}_n\Big) \rightarrow 0
        \end{gathered}
    \end{equation*}
    in $\psingle_{f_0,\sigma_0}^{(n)}$-probability as $n\rightarrow \infty$ for any $M_n \rightarrow \infty$.
\end{theorem}

\begin{proof}
See \Cref{subsec:proof_reg_ss}.
\end{proof}

The contraction rate $\epsilon_n$ is near-minimax rate and only depends on intrinsic smoothness $\tilde{\bfs}$ of the function. \citet{suzuki_deep_2021} showed that the empirical risk minimizer under the squared loss achieves the same rate $\epsilon_n$. Therefore, BNNs attain the same theoretical optimality as their frequentist counterpart for anisotropic Besov spaces.

\begin{remark}\label{rmk:unbounded_ss}
The exponential tail condition in \aref{cond:ss_support} is imposed to effectively control model complexity, as measured by the metric entropy of a suitably chosen sieve \citep{ghosal2007convergence,ghosal2000convergence}. If one instead uses a polynomial-tailed density, the entropy calculation can be bypassed using the Vapnik-Chervonenkis (VC)-dimension technique \citep{bartlett2019nearly}. In that case, the posterior contraction rate is slightly degraded to $n^{-\tilde{s}/(2\tilde{s} + 1)} (\log n)^2$.
See \Cref{appendix:sparse_ubd} for further details.
\end{remark}

\subsubsection{Shrinkage prior}
\label{sec:shr}

The main bottleneck in using spike-and-slab priors lies in the computational burden introduced by the point-mass component of the prior. To address this issue, we extend the results to continuous shrinkage priors, which offer a more computationally efficient alternative. These priors retain the sparsity-promoting nature of spike-and-slab priors while avoiding the complexity of variable-dimensional posterior inference. This leads to scalable and practical Bayesian procedures suitable for high-dimensional settings, where computational tractability is essential.

Using the network parameters in \eqref{eqn:aniso_optimal_hyperparm}, we place a prior over $\Theta(L_{1n}, D_{1n})$. That is, sparsity is not imposed explicitly but is instead induced implicitly through the use of shrinkage priors.
We express a shrinkage prior as
\begin{equation}\label{eqn:shr_prior}
    \pi(\theta \mid L,D) = \prod_{j=1}^{T} \tilde \pi_{SH}(\theta_j),
\end{equation}
where $\tilde \pi_{SH}$ is a density on $\mathbb R$ and $T = \abs{\Theta(L,D)}$. 
Specifically, $\tilde \pi_{SH}$ is assumed to satisfy the following assumptions.
\begin{description}
    \item[(C1)\label{cond:shr_support}] 
     $
     \log \int_{|u| > K_n} \tilde \pi_{SH} (u) d u \lesssim - K_n
     $ for any $K_n \rightarrow \infty$.
    \item[(C2)\label{cond:shr_tail}] 
    $
         \log \inf_{|u|\le B_1} \tilde{\pi}_{SH}(u) \gtrsim -(\log n)^2
    $.
    \item[(C3)\label{cond:shr_spike}] $ \log\int_{|u|> a_{n}} \tilde{\pi}_{SH}(u) du\le-C_A (\log n)^2$ for a sufficiently large constant $C_A>0$,
    where $a_n = e^{-2L_{1n}\log n}$.    
\end{description}

Assumptions~\aref{cond:shr_support}--\aref{cond:shr_tail} serve similar roles to Assumptions~\aref{cond:ss_support}--\aref{cond:ss_tail}. Assumption~\aref{cond:shr_spike} imposes an additional constraint that serves as a continuous analogue of the spike component in spike-and-slab priors. This assumption is indeed restrictive: for rapidly decreasing $a_n$, it requires the prior to concentrate nearly all its mass on $[-a_n, a_n]$. As a result, widely used shrinkage priors such as the horseshoe \citep{Carvalho_Polson_Scott_2010} do not satisfy Assumptions~\aref{cond:shr_support}--\aref{cond:shr_spike} directly. Nevertheless, the assumptions can be satisfied by designing priors that place most of their mass near zero, as required in Assumption~\aref{cond:shr_spike}, while also satisfying the exponential tail decay in Assumption~\aref{cond:shr_support}. The following examples illustrate such constructions, with verification given in Appendices \ref{subsec:proof_ex_relaxed_ss} and \ref{subsec:proof_ex_gaussian_mixture}.

\begin{example}[Relaxed spike-and-slab \citep{lee2022asymptotic}]\label{ex:relaxed_ss}
Let $\varphi_k$ denote a sub-Weibull density with tail index $0 < k \leq 1$ \citep{vladimirova2020sub} satisfying $\int_{|u|>K} \varphi_k(u)\,du \leq C_1 \exp(-C_2 K^{1/k})$ for any $K > 0$ and some constants $C_1,C_2>0$. For $C_u>B_1$,
define $\tilde \pi_{SH}$  as
\begin{align}
\label{eqn:Gmm}
\tilde \pi_{SH}(u) = \pi_{1n} \frac{1}{\sigma_{1n}} \varphi_k\!\left(\frac{u}{\sigma_{1n}} \right) + \pi_{2n} U(u; -C_u,C_u),
\end{align}
where $\pi_{1n} = 1 - e^{-2C_A(\log n)^2}$, $\pi_{2n} = e^{-2C_A(\log n)^2}$, $\sigma_{1n} =a_n(2C_A(\log n)^2)^{-k}$ and a sufficiently large $C_A>0$. Then, $\tilde \pi_{SH}$ satisfies Assumptions~\aref{cond:shr_support}--\aref{cond:shr_spike}. 
\end{example}

\begin{example}[Relaxed spike-and-slab; Gaussian slab]
\label{ex:gaussian_mixture}
    If the uniform component in~\eqref{eqn:Gmm} is replaced by the density of a zero-mean Gaussian distribution with fixed variance, then $\tilde \pi_{SH}$ satisfies Assumptions~\aref{cond:shr_support}--\aref{cond:shr_spike}. 
\end{example}

\begin{remark}
    For $k = 1$, a sub-Weibull density $\varphi_k$ corresponds to a sub-exponential density. For $k = 1/2$, it corresponds to a sub-Gaussian density.
    In particular, if $\varphi_k$ is chosen as a Gaussian density, the prior in \Cref{ex:gaussian_mixture} becomes a Gaussian mixture prior.
\end{remark}
Note that the mixture weight $\pi_{1n}$ in \eqref{eqn:Gmm} approaches 1. Although this may be practically undesirable, it ensures that Assumption~\aref{cond:shr_spike} is satisfied. Similar constraints have also been adopted in the literature on sparse BNNs with shrinkage priors \citep{sun2021consistent,lee2022asymptotic}.
We now formalize the contraction result for shrinkage priors. The following theorem shows that BNNs equipped with shrinkage priors achieve the near-minimax rate.

\begin{theorem}[Shrinkage prior]\label{thm:shrinkage}
Suppose that Assumptions~\aref{assum:a-1}--\aref{assum:a-3} hold, and that the prior distribution in \eqref{eqn:shr_prior} is placed over $\Theta(L_{1n}, D_{1n})$. Assume further that the continuous density $\tilde \pi_{SH}$ satisfies Assumptions~\aref{cond:shr_support}--\aref{cond:shr_spike}.
    Then, the posterior distribution concentrates at the rate $\epsilon_n = n^{-\tilde{\bfs}/(2\tilde{\bfs}+1)}(\log n)^{3/2}$, in the sense that
    \begin{equation*}
        \begin{gathered}
        \Pi\Big((f ,\sigma)\in{\Phi}(L_{1n}, D_{1n})\times [\underline{\sigma}, \overline{\sigma}]: 
        \lVert{f - f_0}\rVert_{L^2(P_X)} + |\sigma^2 - \sigma_0^2|  > M_n \epsilon_n \mid \mathcal{D}_n\Big) \rightarrow 0
        \end{gathered}
    \end{equation*}
    in $\psingle_{f_0,\sigma_0}^{(n)}$-probability as $n\rightarrow \infty$ for any $M_n \rightarrow \infty$.
\end{theorem}

\begin{proof}
See \Cref{subsec:proof_shrinkage}.
\end{proof}

\subsubsection{Rate adaptation}

In \Cref{sec:regression_sas} and \Cref{sec:shr}, the priors depend on the smoothness parameter $s$, indicating that the procedures are not rate-adaptive. By placing suitable priors on the network parameters, the results can be extended to achieve rate adaptation. Therefore, the optimal contraction rate is achieved without knowing the characteristics of the true function.

Specifically, instead of the network parameters in \eqref{eqn:aniso_optimal_hyperparm}, we consider the depth $\tilde{L}_{n}=\lceil C_L \log n \rceil$   for a sufficiently large $C_L>0$ and the priors $\pi_D$ and $\pi_S$ on the width $D$ and sparsity $S$, respectively, given by 
    \begin{equation}
    \label{eqn:adaptive_prior}
        \begin{gathered}
          \pi_D(D) \propto e^{-\lambda_D D (\log D)^3},\quad \pi_S(S) \propto e^{-\lambda_S S (\log S)^2},
        \end{gathered}
    \end{equation}
    for constants $\lambda_D>0$ and $\lambda_S > 0$.  

\begin{theorem}[Adaptation]\label{thm:adaptive_estimation}
  Suppose that Assumptions~\aref{assum:a-1}--\aref{assum:a-3} hold, and that the prior distributions satisfy either one of the following conditions:
  \begin{itemize}[leftmargin=1.8em,itemsep=-0.3em,topsep=0em]
  \item[(i)] The priors $\pi_D$ and $\pi_S$ in \eqref{eqn:adaptive_prior} are assigned to $(D,S)$, respectively, and the prior in \eqref{eqn:sparse_prior} is placed over $\Theta(\tilde{L}_{n}, D, S)$ conditional on $(D, S)$. The slab part $\tilde{\pi}_{SL}$ satisfies Assumptions~\aref{cond:ss_support}--\aref{cond:ss_tail}.
  \item [(ii)] The prior $\pi_D$ in \eqref{eqn:adaptive_prior} is assigned to $D$, and the prior in \eqref{eqn:shr_prior} is placed over $\Theta(\tilde{L}_{n}, D)$ conditional on $D$. The continuous density $\tilde \pi_{SH}$ satisfies Assumptions~\aref{cond:shr_support}--\aref{cond:shr_spike}.
  \end{itemize}
    If $C_L$ is sufficiently large so that $\tilde L_{n}\ge L_{1n}$, then the posterior distribution concentrates at the rate $\epsilon_n = n^{-\tilde{\bfs}/(2\tilde{\bfs}+1)}(\log n)^{3/2}$, in the sense that
    \begin{equation*}
        \begin{gathered}
        \Pi\Big((f ,\sigma)\in{\Phi}(\tilde L_{n})\times [\underline{\sigma}, \overline{\sigma}]: 
        \lVert{f - f_0}\rVert_{L^2(P_X)} + |\sigma^2 - \sigma_0^2|  > M_n \epsilon_n \mid \mathcal{D}_n\Big) \rightarrow 0
        \end{gathered}
    \end{equation*}
    in $\psingle_{f_0,\sigma_0}^{(n)}$-probability as $n\rightarrow \infty$ for any $M_n \rightarrow \infty$. 
    \end{theorem}

\begin{proof}
See \Cref{subsec:proof_adaptive}.
\end{proof}

In both cases, the priors are placed over $\Theta(\tilde{L}_{n})$, and the procedures achieve rate adaptation once $C_L$ is chosen sufficiently large. Although the priors do not depend on the Besov parameters for the true function, the conditions on $\tilde{\pi}_{SL}$ and $\tilde{\pi}_{SH}$ still involve these unknown quantities. Nevertheless, the required conditions can be satisfied without explicit knowledge of them. Specifically, if \Cref{ex:uniform} and \Cref{ex:relaxed_ss} are defined with a sufficiently large constant $C_v$ such that $C_v > B_1$, the conditions hold regardless of the value of $B_1$, which depends on the Besov parameters. In contrast, \Cref{ex:gaussian} and \Cref{ex:gaussian_mixture} are already independent of these parameters.
Therefore, rate adaptation can be readily achieved in practice. 

\subsection{Posterior contraction in composite Besov spaces}\label{subsec:composite}

In this section, we show that BNNs attain the minimax optimal rate over composite Besov classes. In other words, we demonstrate that BNNs outperform traditional statistical approaches by achieving the minimax optimal convergence rate. Following \citet{suzuki_deep_2021}, we assume that the true regression function satisfies one of the following two structural conditions.

\begin{description}
  \item[(A4)\label{assum:a-4}] The true regression function $f_0$ satisfies $f_0 \in \Besov_{p,q}^{\dnaught,\tnaught,\bfsnaught} \cap \UB$ for some $0<p, q \leq \infty$, $\dnaught$, $\tnaught$, and $\bfsnaught$, such that $\tilde{\bfs}^{(1)} > (1/p-1/2)_+$ and $\tilde{\bfs}^{(h)} > 1/p $ for $h =2,\dots,H$.
  \item[(A5)\label{assum:a-5}] The true regression function $f_0$ is defined as $f_0 = f_2 \circ f_1$, where $f_1 = A \cdot + b$ with $A \in \mathbb{R}^{d^{(1)} \times d^{(0)}}$ and $b \in \mathbb{R}^{d^{(1)}}$ such that $d^{(1)} \leq d^{(0)}$ and $f_1(x) \in [0,1]^{d^{(1)}}$ for all $x\in[0,1]^{d^{(0)}}$, and $f_2 \in U(\Besov_{p,q}^{\bfs^{(2)}}) \cap \UB$ with $0<p, q \leq \infty$, and $\bfs^{(2)} \in \mathbb{R}_{++}^{d^{(1)}}$ such that $\tilde{\bfs}^{(2)} > (1/p-1/2)_+$.
\end{description}

Assumption~\aref{assum:a-4} corresponds to the composite Besov space defined in \Cref{sec:compo}. We note that our assumptions are comparable to those in \citet{suzuki_deep_2021}.
Assumption~\aref{assum:a-5} defines a nontrivial Besov class involving an affine transformation, which is a specific case of the general class $\Besov_{p,q}^{\dnaught,\tnaught,\bfsnaught}$. However, it is not contained within Assumption~\aref{assum:a-4} owing to the relaxed smoothness requirement on $\tilde{\bfs}^{(2)}$.
Under Assumption \aref{assum:a-4}, define  $t^{\ast(h)} = \underline{\bfs}^{(h)}/\tilde{\bfs}^{(h)}$, $\tilde{\bfs}^{\ast(h)} = \tilde{\bfs}^{(h)} \prod_{k=h+1}^H \{ (\underline{\bfs}^{(k)} - t^{\ast(k)}/p) \wedge 1 \}$, $h\in[H]$, and $h^\ast = \argmin_{h \in [H]} \tilde{\bfs}^{\ast(h)}$.  We define the intrinsic dimension as $t^\ast = t^{(h^\ast)}$ and the intrinsic smoothness as $\tilde{\bfs}^\ast = \tilde{\bfs}^{\ast(h^\ast)}$. 
For Assumption \aref{assum:a-5}, we adopt the same definitions with $h^\ast = 2$. \citet{suzuki_deep_2021} obtained a lower bound of the minimax risk that roughly matches $n^{-\tilde{\bfs}^\ast/(2\tilde{\bfs}^\ast+1)}$. 
The following examples illustrate practically relevant functions that are covered by either Assumption~\aref{assum:a-4} or Assumption~\aref{assum:a-5}.

\begin{example}[Additive function; \cref{fig:additive}]\label{ex:additive}
    Suppose $f(x) = \sum_{i=1}^d g_i(x_i)$, where $g_i \in U(\Besov_{p,q}^{s_0})$ for $s_0 > (1/p-1/2)_+$. Then, $f$ is covered by Assumption~\aref{assum:a-4} in the form $f \in \Besov_{p,q}^{(d, d, 1), (1, d), (s_0, s^{(2)})}$ for $s^{(2)}$ such that $\tilde{s}^{(2)}$ can be taken arbitrarily large, reflecting the additive structure.
\end{example}

\begin{example}[Multiplicative function; \Cref{fig:multiplicative}]\label{ex:multi}
    Suppose $f(x)=\prod_{i=1}^d g_i(x_i)$, where $g_i\in U(\Besov_{p,q}^{s_0})$ for $s_0>(1/p-1/2)_+$. This $f$ is also covered by Assumption~\aref{assum:a-4}, similar to the additive case in \Cref{ex:additive}.
\end{example}

\begin{example}[Rotation; \Cref{fig:ex_aniso}]\label{ex:rotation}
    Let $R_\tau \in \mathbb{R}^{d \times d}$ be the rotation matrix by angle $\tau$. Suppose  $f = g\circ(A\cdot+b)$, where $g \in U(\Besov_{p,q}^{\bfs^{(2)}})\cap\UB$, $A=R_\tau/\sqrt{d}$ and $b=(I-R_\tau)(1/2, \cdots, 1/2)^T$. This $f$ is covered by Assumption~\aref{assum:a-5}.
  \end{example}

\begin{example}[Piecewise function]\label{ex:piecewise}
    The indicator function of a hyper-rectangle lies in $\Besov_{p,\infty}^{s}$ if $\overline{\bfs} \leq 1/p$ for $1 \leq p < \infty$ \citep{sickel2020regularity}. For hyper-rectangles $A_i\subset [0,1]^d$, suppose $f(x) = \sum_{i=1}^H I(x \in A_i)\, g_i(x)$, where $g_i \in U(\Besov_{p,\infty}^{s_0})$ for $(1/p - 1/2)_+ < \tilde{s}_0$ and $\overline{\bfs}_0 \leq 1/p$.
    This $f$ is covered by Assumption~\aref{assum:a-4} using Examples~\ref{ex:additive} and~\ref{ex:multi}. 
\end{example}

We now establish posterior contraction properties for the composite Besov spaces. Similar to the anisotropic Besov case, our results rely on the existence of an optimal NN approximator $\hat{f} \in \Phi(L_{2n}, D_{2n}, S_{2n}, B_2)$ under Assumption~\aref{assum:a-4}, and $\hat{f} \in \Phi(L_{3n}, D_{3n}, S_{3n}, B_3)$ under Assumption~\aref{assum:a-5}, where 
$$
L_{kn} \asymp \log n,\quad D_{kn} \asymp N_n^\ast,\quad S_{kn} \asymp N_n^\ast \log n,\quad B_{k} \propto 1,\quad k=2,3,
$$
with $N_n^\ast = \lceil n^{1/(2\tilde{\bfs}^{(\ast)}+1)} \rceil$. Therefore, the network parameters share the same asymptotic orders for $k = 2, 3$, although their specific values differ. 
For further details, see \Cref{lem:anisotropic_approx} and \Cref{rmk:hyper_bigo}. The posterior contraction results are formalized below.

\begin{theorem}[Composite anisotropic Besov]\label{thm:general_function}
Suppose that Assumptions~\aref{assum:a-2}--\aref{assum:a-3} hold, and that either Assumption~\aref{assum:a-4} or Assumption~\aref{assum:a-5} is satisfied as follows.
  \begin{itemize}[leftmargin=1.8em,itemsep=-0.3em,topsep=0em]
  \item[(i)] If Assumption~\aref{assum:a-4} holds, assume that the priors and conditions in one of Theorems~\ref{thm:reg_ss}, \ref{thm:shrinkage}, and \ref{thm:adaptive_estimation} are satisfied with $(L_{2n},D_{2n},S_{2n})$ in place of $(L_{1n},D_{1n},S_{1n})$.
  \item [(ii)] If Assumption~\aref{assum:a-5} holds, assume that the priors and conditions in one of Theorems~\ref{thm:reg_ss}, \ref{thm:shrinkage}, and \ref{thm:adaptive_estimation} are satisfied with $(L_{3n},D_{3n},S_{3n})$ in place of $(L_{1n},D_{1n},S_{1n})$.
  \end{itemize}
Then, the posterior distribution concentrates at the rate $\epsilon_n = n^{-\tilde{\bfs}^\ast/(2\tilde{\bfs}^\ast+1)}(\log n)^{3/2}$.
\end{theorem}

\begin{proof}
See \Cref{subsec:proof_general}.
\end{proof}

The contraction rate in \Cref{thm:general_function} matches that of \citet{suzuki_deep_2021}, indicating that BNNs attain the same level of optimality. However, achieving the optimal rate in \citet{suzuki_deep_2021} requires a correctly specified network structure. In contrast, our Bayesian procedure offers a clear advantage by learning the unknown compositional depth, connectivity, and anisotropic smoothness directly from data, without requiring prior structural knowledge.

\section{Discussion}\label{sec:conclusions}

In this paper, we have established that sparse BNNs achieve optimal posterior contraction rates over anisotropic Besov spaces and their hierarchical compositions. We show that BNNs equipped with either spike-and-slab or continuous shrinkage priors attain a near-minimax rate that depends only on the intrinsic dimension, thereby overcoming the curse of dimensionality. Furthermore, BNNs are shown to achieve rate adaptation over both anisotropic and composite Besov classes.

\textbf{Practice meets theory.}
Our work provides a rigorous theoretical foundation for the empirical success of NNs on complex real-world data. By establishing optimal rates over anisotropic Besov spaces, we shed light on why NNs perform well in high-dimensional settings where classical smoothness-based theory falls short. Although our analysis focuses on BNNs, which additionally offer rate adaptation and uncertainty quantification, the insights extend more broadly to understanding NN performance. The results provide actionable guidance: practitioners dealing with such data may benefit from employing sparse architectures with sparsity-inducing priors, potentially inspiring new algorithmic advances for efficient inference.

\textbf{Extension.} Our theoretical results are established under the Gaussian nonparametric regression model in \eqref{eqn:reg:model}. These results are readily extended to other statistical models in which the Hellinger distance can be translated into an $L^2$-type distance for the underlying function. As one such extension, we present posterior contraction results for nonparametric binary classification. We also show that the same contraction rates derived for the $L^2$-norm with respect to $P_X$ hold for the empirical $L^2$-norm by applying empirical process theory. See \Cref{appendix:asymptotic_iid} for details.

\textbf{Future work.} This study has several limitations that suggest promising directions for future research. While the spike-and-slab prior offers desirable theoretical properties, its practical implementation is hampered by the point mass at zero, which introduces substantial computational challenges. Variational approximations may partially alleviate this burden \citep{cherief2020convergence, bai2020efficient}. Shrinkage priors address the issue differently, but our theoretical framework does not accommodate widely used choices such as the horseshoe prior \citep{Carvalho_Polson_Scott_2010}. Extending our results to incorporate such popular shrinkage priors is an important direction. 
Furthermore, a natural direction for future work is to extend our theoretical framework to modern deep learning architectures widely used in practice. Recent developments in the statistical theory of convolutional NNs \citep{kohler2022convergence,fang2023optimal} and transformers \citep{kim2024transformers,jiao2025approximation} offer valuable foundations for establishing posterior contraction results beyond fully connected architectures.

\section*{Acknowledgment}
This research was supported by grants from the National Research Foundation of Korea (NRF) funded by the Korean government (MSIT) (2022R1C1C1006735, RS-2023-00217705, RS-2025-00513129) and the National Science Foundation (NSF) (DMS-2503119).

\bibliographystyle{myabbrvnat}
\bibliography{refs}



\newpage

\renewcommand{\theequation}{S\arabic{equation}}


\appendix

\setcounter{equation}{0}

\section{Supplement on posterior contraction}\label{appendix:asymptotic_iid}

\subsection{Nonparametric regression}

In this section, we present theorems that apply to general settings with independent and identically distributed (IID) data. For a semi-metric space $(A,\rho)$, we let $\Covering{\epsilon, A, \rho}$ denote the $\epsilon$-covering number.
The following result builds on the foundational work of \citet{ghosal2000convergence} and \citet{ghosal2007convergence}.

\begin{lemma}\label{lem:consistency_unknown_var}
    Consider model \eqref{eqn:reg:model} with $f_0\in\mathcal{UB}$ and $\sigma_0 \in[\underline \sigma,\overline \sigma]$ for constants $0 < \underline{\sigma} \leq \overline{\sigma}$. 
    The prior for $\sigma$ is supported on $[\underline{\sigma}, \overline{\sigma}]$.
    For $\mathcal F\subset \mathcal {UB}$, define 
      \begin{align*}
        A_{\epsilon} = \left\{(f, \sigma) \in \mathcal{F}\times[\underline{\sigma},\overline{\sigma}]: \norm{f - f_0}_{L^2(P_X)} \leq \frac{\epsilon}{2}, \abs{\sigma - \sigma_0} \leq \frac{\epsilon}{2} \right\}.
    \end{align*}
Suppose there exist a subset $\mathcal{F}_n \subset \mathcal{F}$ and a sequence $\epsilon_n \rightarrow 0$ with $n \epsilon_n^2 \rightarrow \infty$ such that
    \begin{align}
        \log \Covering{\epsilon_n, \mathcal{F}_n, \norm{\cdot}_{L^2(P_X)}} &\lesssim n\epsilon_n^2,\label{eqn:lem_con1} \\
        -\log \Pi(A_{\epsilon_n}) &\lesssim n \epsilon_n^2,\label{eqn:lem_con2} \\
        \sup_{\sigma \in [\underline{\sigma}, \overline{\sigma}]} \Pi(\mathcal{F} \setminus \mathcal{F}_n \mid \sigma) &= o\!\left(e^{-Cn\epsilon_n^2}\right), \label{eqn:lem_con3}
    \end{align}
    for a sufficiently large $C>0$. Then, the posterior satisfies 
    $$\Pi\!\left((f, \sigma) \in \mathcal{F}\times[\underline{\sigma},\overline{\sigma}] : \norm{f - f_0}_{L^2(P_X)} + |\sigma^2 - \sigma_0^2| > M_n \epsilon_n \mid \mathcal{D}_n\right) \rightarrow 0$$
    in $P_{f_0,\sigma_0}^{(n)}$-probability as $n\rightarrow \infty$ for any $M_n \rightarrow \infty$.
\end{lemma}

\begin{proof}
    See \Cref{subsec:lem:consistency_unknown_var}.
\end{proof}

We focus on the case where $\Pi(\mathcal{F} \setminus \mathcal{F}_n \mid\sigma)$ is invariant with respect to $\sigma$, and we write it simply as $\Pi(\mathcal{F} \setminus \mathcal{F}_n)$.

\subsection{Nonparametric binary classification}\label{appendix:binary_cls}

Suppose we have a set of $n$ input-output observations, where each pair is an independent random sample from a binary classification model with a $d$-dimensional input variable $X_i \in [0, 1]^d$ and an output variable $Y_i \in \Set{0, 1}$:
\begin{align}
\label{eqn:model_cls}
    \psingle(Y_i = 1 \mid  X_i) = 1 - \psingle(Y_i = 0 \mid  X_i) = (\psi_M \circ f_0)(X_i),\quad X_i \sim P_X, \quad i \in [n],
\end{align}
where $\psi_M$ is the sigmoid function $\psi_M(x) = (1+e^{-Mx})^{-1}$ for a large $M>0$ and $f_0:[0,1]^d\rightarrow \mathbb R$ denotes the true regression function.
The scaling factor $M$ is introduced because $f_0$ and the prior support are restricted to $\mathcal{UB}$, whose elements have unit sup-norm.
Let $P_{f}^{(n)}$ denote the joint distribution of $\mathcal{D}_n$ with a regression function $f$.

\begin{lemma}\label{lem:consistency_classification}
    Consider model \eqref{eqn:model_cls} with $f_0\in\mathcal{UB}$. For $\mathcal F\subset \mathcal {UB}$, define 
    \begin{align*}
        A_{\epsilon}^\prime = \left\{f \in \mathcal{F}: \norm{f - f_0}_{L^2(P_X)} \leq \epsilon\right\}.
    \end{align*}
Suppose there exist a subset $\mathcal{F}_n \subset \mathcal{F}$ and a sequence $\epsilon_n \rightarrow 0$ with $n \epsilon_n^2 \rightarrow \infty$ such that
   \begin{align*}
        \log \Covering{\epsilon_n, \mathcal{F}_n, \norm{\cdot}_{L^2(P_X)}} &\lesssim n\epsilon_n^2, \\
        -\log \Pi(A^\prime_{\epsilon_n}) &\lesssim n \epsilon_n^2, \\
        \Pi(\mathcal{F} \setminus \mathcal{F}_n) &= o\!\left(e^{-Cn\epsilon_n^2}\right),
    \end{align*}
    for a sufficiently large $C>0$. Then, the posterior satisfies 
     $$\Pi\!\left(f \in \mathcal{F} : \norm{(\psi_M \circ f) - (\psi_M \circ f_0)}_{L^2(P_X)} > M_n\epsilon_n \mid \mathcal{D}_n\right) \rightarrow 0$$
    in $P_{f_0}^{(n)}$-probability as $n\rightarrow \infty$ for any $M_n \rightarrow \infty$.
\end{lemma}

\begin{proof}
    See \Cref{subsec:proof_lem:consistency_classification}.
\end{proof}

The following theorem shows that a plug-in classifier can achieve results analogous to those in \Cref{thm:reg_ss,thm:shrinkage,thm:adaptive_estimation,thm:general_function} for classification problems, under the same assumptions and notation as in \Cref{sec:main_results}. The proof closely parallels that of the Gaussian case and is therefore omitted.

\begin{theorem}[Nonparametric classification]\label{thm:cls}
    Consider model \eqref{eqn:model_cls} with the same prior specifications and assumptions as in \Cref{thm:reg_ss} (or in \Cref{thm:shrinkage,thm:adaptive_estimation,thm:general_function}), except for Assumption~\aref{assum:a-3}.
    Then, the posterior distribution concentrates at the rate $\epsilon_n = n^{-\tilde{\bfs}/(2\tilde{\bfs}+1)}(\log n)^{3/2}$, in the sense that
    \begin{equation*}
        \Pi\!\left(f \in {\Phi}: \norm{(\psi_M \circ f) - (\psi_M \circ f_0)}_{L^2(P_X)} > M_n \epsilon_n \mid \mathcal{D}_n\right) \rightarrow 0
    \end{equation*}
    in $P_{f_0}^{(n)}$-probability as $n\rightarrow \infty$ for any sequence $M_n \rightarrow \infty$, where {${\Phi}$} denotes the corresponding NN model space defined in the referenced theorem.
\end{theorem}

It can be shown that the minimax rate is also attained for the excess risk for the misclassification error, defined as
    \begin{align*}
    \mathcal{E}(f_0, f) := \mathbb{E}_{f_0}[I(Y \neq I((\psi_M \circ f)(X) \ge 1/2))] - \mathbb{E}_{f_0}[I(Y \neq I((\psi_M \circ f_0)(X) \ge 1/2))],
\end{align*}
implying that optimal performance is achieved in general nonparametric classification settings.
It is well known that \citep{yang1999minimax} the excess risk satisfies
\begin{align*}
    \mathcal{E}(f_0, f) \leq 2 \norm{(\psi_M \circ f) - (\psi_M \circ f_0)}_{L^2(P_X)}.
\end{align*}
This provides the posterior contraction rate with respect to the misclassification error  as follows.
 
\begin{corollary}[Nonparametric classification with respect to the misclassification error]\label{thm:cls_misclass}
    Consider model \eqref{eqn:model_cls} with the same prior specifications and assumptions as in \Cref{thm:reg_ss} (or in \Cref{thm:shrinkage,thm:adaptive_estimation,thm:general_function}), except for Assumption~\aref{assum:a-3}.
    Then, the posterior distribution concentrates at the rate $\epsilon_n = n^{-\tilde{\bfs}/(2\tilde{\bfs}+1)}(\log n)^{3/2}$, such that
    \begin{equation*}
    \Pi\!\left(f \in {{\Phi}}: \mathcal{E}(f_0, f) > M_n \epsilon_n \mid \mathcal{D}_n\right) \rightarrow 0
    \end{equation*}
    in $P_{f_0}^{(n)}$-probability as $n\rightarrow \infty$ for any sequence $M_n \rightarrow \infty$, where {${\Phi}$} denotes the corresponding NN model space defined in the referenced theorem.
\end{corollary}

\begin{remark}
This result can be extended to achieve optimal convergence rates under boundary conditions, particularly under \textit{Tsybakov's margin condition} \citep{tsybakov_optimal_2004}, as demonstrated in frequentist approaches \citep{kim_fast_2021}. However, incorporating the likelihood corresponding to this loss function poses substantial challenges, and we leave this extension for future work.
\end{remark}

\subsection{Empirical norm}\label{appendix:empirical_norm}

The following lemma describes the relationship between expectations of empirical processes.
\begin{lemma}[Theorem 19.3 of \citet{gyorfi2002distribution}]\label{lem:gyorfi2002}
    Let $X,~X_1,\cdots, X_n \in \mathbb{R}^d$ be IID random vectors drawn from the distribution $P_X$. Let $K_1,~K_2 \ge 1$ be constants and let $\mathcal{G}$ be a class of functions $g: \mathbb{R}^d \rightarrow \mathbb{R}$ such that
    \begin{equation*}
        \abs{g(x)} \leq K_1, \quad \mathbb{E}[g(X)^2] \leq K_2 \mathbb{E}[g(X)].
    \end{equation*}
    For $0 < \alpha < 1$ and $\epsilon > 0$, assume that
    \begin{equation*}
        \sqrt{n}\alpha \sqrt{1-\alpha}\sqrt{\epsilon} \ge 288 \max\Set{2K_1, \sqrt{2K_2}},
    \end{equation*}
    and that, for all $x_1, \cdots, x_n \in \mathbb{R}^d$ and for all $t \ge \epsilon/8$,
    \begin{equation*}
        \begin{aligned}
        \frac{\sqrt{n}\alpha (1 - \alpha) t}{96\sqrt{2} \max\Set{K_1, 2K_2}} 
        \ge \int_{\frac{\alpha(1-\alpha)t}{16\max\Set{K_1, 2K_2}}}^{\sqrt{t}} \!\sqrt{\log \Covering{u, \Set{g \in \mathcal{G}: \frac{1}{n} \sum_{i=1}^n g(x_i)^2 \leq 16t}, \norm{\cdot}_{1, n}}} du,
        \end{aligned}
    \end{equation*}
    where $\norm{g}_{1,n} = \frac{1}{n} \sum_{i=1}^n \abs{g(X_i)}$.
    Then, 
    \begin{equation*}
        \Pr\!\left( \sup_{g \in \mathcal{G}} \frac{\abs{\mathbb{E}[g(X)] - \frac{1}{n}\sum_{i=1}^n g(X_i)}}{\epsilon + \mathbb{E}[g(X)]} > \alpha \right) \leq 60 \Exp{- \frac{n\epsilon \alpha^2(1-\alpha)}{128 \cdot 2304 \max\Set{K_1^2, K_2}}}.
    \end{equation*}
\end{lemma}

For a function $f:[0,1]^d\rightarrow\mathbb R$, we define the empirical $L^2$-norm as $\norm{f}_n=(n^{-1}\sum_{i=1}^n |f(X_i)|^2)^{1/2}$. 
As a corollary of \Cref{lem:gyorfi2002}, we present the following result, adapted from \citet{kong2023masked}, which establishes a relationship between the empirical norm $\norm{\cdot}_{n}$ and the population norm $\norm{\cdot}_{L^2(P_X)}$.

\begin{lemma}[Change of norm]\label{lem:change_norm}
    Let $N_n = \lceil n^{1/(2\tilde{\bfs}+1)} \rceil,~\epsilon_n = n^{-\tilde{\bfs}/(2\tilde{\bfs}+1)} (\log n)^{3/2}$, and ${{\Phi}}$ be an NN model space in \Cref{thm:reg_ss} (or in \Cref{thm:shrinkage,thm:adaptive_estimation,thm:general_function}). Then, 
    \begin{equation*}
        P_{f_0,\sigma_0}^{(n)}\!\left( \sup_{f \in \Phi} \frac{\abs{\norm{f - f_0}_{L^2(P_X)}^2 - \norm{f - f_0}_{n}^2}}{M_0 \epsilon_n^2 + \norm{f - f_0}_{L^2(P_X)}^2} > \frac{1}{2}  \right) \leq 60 \Exp{- \frac{n M_0 \epsilon_n^2}{8 \cdot 128 \cdot 2304 \cdot 16}}
    \end{equation*}
    holds for a sufficiently large constant $M_0 > 0$ and sufficiently large $n$.
\end{lemma}

\begin{proof}
    See \Cref{subsec:proof_change_norm}.
\end{proof}

Using \Cref{lem:change_norm}, we deduce that for any $f \in \Phi$ and all sufficiently large $n$, the following inequalities hold with probability $1 - \Exp{-M_0^\ast n\epsilon_n^2}$, for some $M_0^\ast > 0$:
\begin{align*}
\norm{f - f_0}_{n}^2 &\leq \frac{3}{2} \norm{f - f_0}_{L^2(P_X)}^2 + \frac{M_0}{2} \epsilon_n^2, \\
\norm{f - f_0}_{L^2(P_X)}^2 &\leq 2 \norm{f - f_0}_{n}^2 + M_0 \epsilon_n^2.
\end{align*}
This yields a contraction rate with respect to the empirical norm.

\begin{theorem}[Contraction rate with respect to the empirical norm]\label{thm:emp_norm}
    Consider model \eqref{eqn:reg:model}, prior and the assumptions as in \Cref{thm:reg_ss} (or in \Cref{thm:shrinkage,thm:adaptive_estimation,thm:general_function}).
    Then, the posterior distribution concentrates at the rate $\epsilon_n = n^{-\tilde{\bfs}/(2\tilde{\bfs}+1)}(\log n)^{3/2}$, in the sense that
    \begin{equation*}
        \Pi\!\left((f, \sigma) \in {\Phi} \times [\underline{\sigma}, \overline{\sigma}]: \norm{f - f_0}_{n} + |\sigma^2 - \sigma_0^2| > M_n \epsilon_n \mid \mathcal{D}_n\right) \rightarrow 0
    \end{equation*}
    in $P_{f_0,\sigma_0}^{(n)}$-probability as $n\rightarrow \infty$ for any sequence $M_n \rightarrow \infty$, where {$\Phi$} denotes the corresponding NN model space defined in the referenced theorem.
\end{theorem}

\begin{proof}
    See \Cref{subsec:proof_reg_empirical}.
\end{proof}

\section{Properties of ReLU networks for Besov spaces}\label{appendix:consistency_NN}

This section describes key properties of ReLU networks, highlighting their capacity to approximate a broad class of functions. In particular, ReLU networks exhibit the universal approximation property: with sufficient depth or width, they can approximate any anisotropic Besov function on $[0,1]^d$ to arbitrary accuracy. We develop the theoretical foundations of this property and discuss its implications for both the design and theoretical understanding of deep learning models. Let $\tilde{\Phi}(\cdot)=\Set{f_\theta:\theta\in\Theta(\cdot)}$ denote the NN model space without the $\mathrm{clip}$ function.

For the first step, note that a ReLU NN can approximate a cardinal B-spline with arbitrary precision \citep{suzuki_deep_2021,suzuki2018adaptivity}. Define $\BSBS(x) = 1$ for $x \in [0,1]$ and $\BSBS(x) = 0$ otherwise. The cardinal B-spline of order $m$ is obtained recursively by convolution:
\begin{equation*}
    \BSBS_m(x) = \begin{cases}
        (\BSBS \ast \BSBS_{m-1})(x) & \text{for } m > 0, \\
        \BSBS(x) & \text{for } m = 0,
    \end{cases}
\end{equation*}
where $(f \ast g)(x) := \int f(x - t)g(t)dt$. For $k \in \mathbb{Z}_+$ and $j = (j_1, \ldots, j_d) \in \mathbb{Z}_+^d$, define
\begin{equation*}
M^{d,m}_{k,j}(x) = \prod_{i=1}^d \BSBS_m(2^{\lfloor ks_i^{-1} \rfloor} x_i - j_i),
\end{equation*}
for $x \in \mathbb{R}^d$. 
Here, $\bfs = (s_1, \ldots, s_d) \in \mathbb{R}_{++}^d$ denotes the smoothness parameter, $k$ controls the spatial resolution, and $j$ determines the location at which the basis function is centered. A function $f$ in an anisotropic Besov space can thus be approximated by a superposition of $M^{d,m}_{k,j}(x)$, which is closely related to wavelet basis functions \citep{mallat1999wavelet}. The following lemma establishes the approximation of the cardinal B-spline basis by ReLU activations.

\begin{lemma}[Approximation of B-spline basis by NNs; Lemma 1 of \citet{suzuki_deep_2021}]\label{lem:approx_BSbyNN}
There exists a constant $c(d,m)$, depending only on $d$ and $m$, and an NN $\hat{M} \in \tilde{\Phi}(L_0, D_0, S_0, B_0)$ with 
\begin{align*}
L_0 &:= 3 + 2\left\lceil\log_2\left(\frac{3^{d\vee m}}{\epsilon c(d,m)}\right) + 5\right\rceil\left\lceil\log_2(d \vee m)\right\rceil, \\
D_0 &:= 6dm(m+2)+2d, \\
S_0 &:= L_0D_0^2, \\
B_0 &:= 2(m + 1)m
\end{align*}
that satisfies $\hat{M}(x) = 0$ for all $x \notin [0, m + 1]^d$ and
$\|M^{d,m}_{0,0} - \hat{M}\|_{\infty} \leq \epsilon$ for all $\epsilon > 0$.
\end{lemma}

For order $m \in \mathbb{N}$ of the cardinal B-spline bases, define
\begin{equation*}
    J(k) := \prod_{j=1}^d J_j(k),\quad J_i(k) = \{-m, -m+1, \cdots, 2^{\lfloor k s_i^{-1}\rfloor}\},
\end{equation*}
and define the quasi-norm (\textit{Besov sequence norm}) of the sequence $(\alpha_{k,j})_{k \in \mathbb{Z}_+,j \in J(k)}$ as 
\begin{equation}\label{eqn:besov_seq_norm}
    \left\| \left( \alpha_{k,j} \right)_{k,j} \right\|_{b^{\bfs}_{p,q}} = 
    \left\{
    \sum_{k=0}^{\infty} 
    \left[
    2^{k \left[ \underline{\bfs} - \left( \sum_{i=1}^{d} \lfloor k s_i^{-1} \rfloor / k \right) / p \right]}
    \left( 
    \sum_{j \in J(k)} \left| \alpha_{k,j} \right|^p 
    \right)^{1/p}
    \right]^q
    \right\}^{1/q}
    .
\end{equation}
For $p=\infty$ or $q=\infty$, the norm is modified in the standard way.

As noted in Remark~\ref{rmk:besov-embeddings}, anisotropic Besov spaces admit useful continuous embeddings.

\begin{remark}[Continuous embeddings in anisotropic Besov spaces; Proposition 1 of \citet{suzuki_deep_2021}]
  \label{rmk:besov-embeddings}
  Let $\bfs = (s_1, \cdots, s_d)  \in \mathbb{R}_{++}^d$, Then, the following continuous embeddings hold on a bounded domain $\Omega\subset\mathbb{R}^d$.
  \begin{itemize}[leftmargin=1.8em,itemsep=-0.3em,topsep=0em]
    \item If $s_1=\cdots=s_d = s_0\notin\mathbb{N}$, then
    $C^{s_0}(\Omega)\;=\;\Besov_{\infty,\infty}^{(s_0,\dots,s_0)}(\Omega).$
    \item For $0<p_1\leq p_2\leq\infty$, $0<q\leq\infty$, if $\tilde{\bfs} > \bigl(1/{p_1}-1/{p_2}\bigr)_+,$
    then
    $\Besov_{p_1,q}^{\bfs}(\Omega) \;\hookrightarrow\;
      \Besov_{p_2,q}^{\gamma\,\bfs}(\Omega),~ \gamma \;=\; 1 \;-\; (1/p_1 - 1/p_2)_+/{\tilde{\bfs}}.$
    \item For $0<p\leq\infty$ and $0<q_1<q_2\leq\infty$,
    $\Besov_{p,q_1}^{\bfs}(\Omega)\;\hookrightarrow\;\Besov_{p,q_2}^{\bfs}(\Omega)$.
    \item If $0<p,q\leq\infty$ and $\tilde{\bfs}>1/p$, $\Besov_{p,q}^{\bfs}(\Omega)\;\hookrightarrow\;C^0(\Omega)$.
  \end{itemize}
The notation $\hookrightarrow$ denotes a continuous embedding from the space on the left into the space on the right. In particular, if $\tilde{\bfs} > 1/p$, then
$\Besov_{p,q}^{\bfs}(\Omega) \;\hookrightarrow\; C^{\gamma\underline{\bfs}}(\Omega)$ with $\gamma = 1-1/({\tilde{\bfs}p})$. In this case, every function in $\Besov_{p,q}^{\bfs}(\Omega)$ is continuous. Otherwise, the functions may display discontinuities or non-smooth behavior such as jumps or spikes.
\end{remark}
Building on these embedding properties, the following lemma establishes the approximation of anisotropic Besov functions by a cardinal B-spline basis, thereby extending the classical spline approximation results for isotropic Besov spaces \citep{suzuki2018adaptivity} to the anisotropic setting.

\begin{lemma}[Approximation of anisotropic Besov function by cardinal B-spline basis; Lemma 2 of \citet{suzuki_deep_2021}]\label{lem:approx_BFbyBS}
    Suppose that $0<p, q,r \leq \infty$, $\omega:=(1/p-1/r)_+ <\tilde{\bfs}$, and $\nu = (\tilde{\bfs}-\omega)/(2\omega)$. 
    Assume that $0 < \overline{\bfs} < \min\{m, m-1+1/p\}$, where $m \in \mathbb{N}$ is the order of the cardinal $B$-spline bases. 
    For any $f \in \Besov^{\bfs}_{p,q}$ and $N>0$, define $K(N)$ such that $2^{\sum_{i=1}^d \lfloor K(N) \underline{s} / s_j\rfloor} = N$ and $K^\ast(N)=\lceil K(N)(1+1/\nu) \rceil$.
    Then, there exist $E_N \subset \Set{(k,j): 1 \leq k \leq K^\ast(N), j \in J(k)}$ with $\abs{E_N} \leq N$ and $f_N$ 
    such that
    \begin{align*}
        f_N(x) &= \sum_{(k,j) \in E_N} \alpha_{k,j} M_{k,j}^{d,m}(x), \\
        \|f - f_N\|_{L^r} &\lesssim N^{-\tilde{\bfs}} \|f\|_{\Besov^{\bfs}_{p,q}},
    \end{align*}
    where the coefficient $(\alpha_{k,j})$ yields the following norm equivalence
\begin{align}\label{eqn:besov_norm_eqv}
        \|f\|_{\Besov_{p,q}^\bfs} \simeq  \|(\alpha_{k,j})_{k,j}\|_{b_{p,q}^\bfs}.
    \end{align}
\end{lemma}

The following lemma extends Proposition 2 of \citet{suzuki_deep_2021}, establishing that for any function $f_0$ in an anisotropic Besov space, there exist NNs that approximate $f_0$ with arbitrary accuracy. In this version, the dependence of the hyperparameter $B$ on $N$ in the original proposition is removed.

\begin{lemma}[Approximation of anisotropic Besov functions by NNs]\label{lem:anisotropic_approx}
    Suppose that $0<p, q,r \leq \infty$, $\omega:=(1/p-1/r)_+ <\tilde{\bfs}$ and $\nu = (\tilde{\bfs}-\omega)/(2\omega)$. 
    Assume that $N \in \mathbb{N}$ is sufficiently large and $m \in \mathbb{N}$ satisfies $0<\overline{\bfs}<\min\{m, m-1+1/p\}$. Then, 
    \begin{equation*}
        \sup_{ f_0 \in U(\Besov_{p,q}^{\bfs}([0,1]^d)) \cap \UB}\inf_{f \in \Phi(L_1,D_1,S_1,B_1)} \|f_0 - f\|_{L^r} \lesssim N^{-\tilde{\bfs}}
    \end{equation*}
   with 
    \begin{equation}\label{eqn:aniso_hparam}
    \begin{aligned}
        L_1 = L_1(d,m,p, r, s,N) &:= 3 + 3\left\lceil\log_2\left(\frac{3^{d\vee m}}{\epsilon c{(d,m)}}\right) + 5\right\rceil\left\lceil\log_2(d\vee m)\right\rceil,\\
        D_1 = D_1(d,m,N) &:= ND_0,\\
        S_1 = S_1(d,m,p, r, s, N) &:= L_1D_0^2 N, \\
        B_1 = B_1(d,m,p,r,\bfs) &:= B_0\exp\left({\frac{(1+ \nu^{-1})[(1/\underline{s})\vee(1/p - \tilde{\bfs})_+]\log 2}{\lceil \tilde{\bfs} \log_2 (d \vee m) \rceil}}\right) 
    \end{aligned}
    \end{equation}
    where $D_0 = D_0(d,m) := 6dm(m + 2) + 2d$, $B_0 = B_0(m):= 2(m + 1)m$, $\epsilon = N^{-\tilde{\bfs} - (1+\nu^{-1})(1/p-\tilde{s})_+}\log(N)^{-1}$, and constant $c{(d,m)}$ does not depend on $N$.
\end{lemma}

\begin{proof}
    See \Cref{subsec:proof_anisotropic_approx}
\end{proof}

\begin{remark}
The key strength of \Cref{lem:anisotropic_approx} lies in that the approximation accuracy depends not on the ambient dimension $d$ but on the intrinsic dimension. When NN models are employed, if the target function exhibits anisotropic smoothness, the approximation rate is governed by the intrinsic rather than the ambient dimension. This leads to minimax-optimal performance in estimation theory and illustrates how NNs can mitigate the curse of dimensionality. In the isotropic case where $\bfs = (s_0, s_0, \dots, s_0)$, we have $\tilde{\bfs} = s_0 / d$, recovering the classical results for ordinary Besov functions in \citet{suzuki2018adaptivity}.
\end{remark}

\begin{remark}\label{rmk:hyper_bigo}
Note that the NN space $\Phi(L, D, S, B)$ is monotonically increasing with respect to $L$, $D$, $S$, and $B$. Thus, the parameters in \eqref{eqn:aniso_hparam} can be replaced by
\begin{equation*}
L_1 = C_L \log N, \quad D_1=C_D N, \quad S_1=C_S N \log N, \quad B_1 = C_B
\end{equation*}
for sufficiently large constants $C_L$, $C_D$, $C_S$, and $C_B$.
\end{remark}

NN models with ReLU activation permit redistribution of weights across layers without altering the output function. The following lemma formalizes this property, which will be used to redistribute the magnitudes of the parameters.

\begin{lemma}[Re-scaling Lemma for ReLU NNs; Lemma A.1 of \citet{kong2024posterior}]\label{lem:re_scaling}
    Let $L \ge 3$ and $D \ge 3$. For positive constants $c_1, \cdots, c_{L+1}$, we define 
    \begin{equation*}
        \tilde{W}^{(l)}:= c_l W^{(l)},~ \tilde{b}^{(l)} = \left(\prod_{l^\prime = 1}^l c_{l^\prime} \right) b^{(l)}
    \end{equation*}
    for $l \in [L+1]$ and define $\tilde{\theta} = (\tilde{W}^{(1)},\tilde{b}^{(1)}, \dots, \tilde{W}^{(L+1)},\tilde{b}^{(L+1)})$. If $\prod_{l=1}^{L+1} c_l = 1$, $f_\theta(x) = f_{\tilde{\theta}}(x)$
    holds for every $x$.
\end{lemma}

\begin{remark}\label{rmk:clipping}
The clipping function can be implemented using a simple ReLU network. Specifically,
    $$
    \mathrm{clip}(x) = \max\Set{\min\Set{x, 1}, -1} =  \zeta\left({x} + 1\right) - \zeta\left({x} - 1\right)  - 1,
    $$
    where $\zeta(\cdot)$ denotes the ReLU activation.
 Consequently, the clipped network space $\Phi(\cdot)$ is contained in the unrestricted network space $\tilde{\Phi}(\cdot)$.
Furthermore, to ensure that each partial block of a composite Besov function produces outputs within $[0,1]$, we apply
    \begin{equation*}
        \mathrm{clip}_{[0,1]}(x) := \zeta(x) - \zeta(x-1) = \max\{x, 0\} - \max\{x-1, 0\}.
    \end{equation*}
With this construction, ReLU networks are capable of approximating composite Besov functions.
\end{remark}

The following lemma extends Theorem 1 of \citet{suzuki_deep_2021}, which establishes an approximation bound for any $f_0$ in a composite anisotropic Besov space by NNs. Unlike their setting, which imposes no restriction on the input dimension of each component, we adopt the composite H\"older framework of \citet{schmidt2020nonparametric} and assume that each layer map depends only on a small subset of its inputs. By combining this effective low-dimensionality constraint with anisotropic smoothness, we exploit both compositional structure and reduced intrinsic dimension to efficiently approximate functions in high ambient dimensions.
\begin{lemma}\label{lem:composite_anisotropic_approx}
    Suppose that $\tilde{\bfs}^{(1)} > (1/p-1/r)_+$ for some $r > 0$ and $\tilde{\bfs}^{(h)} > 1/p $ for all $h \ge 2$. Then, the approximation error over the composite anisotropic Besov space is bounded as
    \begin{equation*}
         \sup_{f_0 \in \Besov_{p,q}^{\dnaught,\tnaught,\bfsnaught}}\inf_{f \in \Phi(L_2,D_2,S_2,B_2)} \norm{f_0 - f}_{L^r} \lesssim N^{-\tilde{\bfs}^{\ast}},   
    \end{equation*}
    where $\tilde{\bfs}^{\ast(h)} = \tilde{\bfs}^{(h)} \prod_{k=h+1}^H \{ (\underline{\bfs}^{(k)} - t^{\ast(k)}/p) \wedge 1 \}$, $\tilde{\bfs}^{\ast} = \min_{h \in [h]} \tilde{\bfs}^{\ast(h)}$ and 
    \begin{equation*}
        \begin{aligned}
            L_2 &= \sum_{h=1}^H \left\{L_1(t^{(h)},m^{(h)}, p, r^{(h)}, \bfs^{(h)}, N) + 1 \right\}, \\
            D_2 &= \max_{h \ge 1} \{ D_1(t^{(h)},m^{(h)}, N) \vee t^{(h+1)} \},\\
            S_2 &= \sum_{h=1}^H d^{(h)} \{ S_1(t^{(h)}, m^{(h)}, p, r^{(h)}, \bfs^{(h)}, N) + 4 \},\\
            B_2 &= \max_{h \ge 1} B_1(t^{(h)},m^{(h)},p, r^{(h)}, \bfs^{(h)}).
        \end{aligned}
    \end{equation*}
    Here, $r^{(1)}=r$, $r^{(h)}=\infty$ for $h \ge 2$, and $m^{(h)} \in \mathbb{N}$ satisfies $0 < \tilde{s}^{(h)} < \min\{m^{(h)}, m^{(h)}-1+1/p\}$.
\end{lemma}

\begin{proof}
    See \Cref{subsec:proof_lem:composite_anisotropic}.
\end{proof}

\begin{lemma}[Theorem 6 of \citet{suzuki_deep_2021}]\label{lem:affine_anisotropic_approx}
Suppose that $d^{(1)}\leq d^{(0)}$ and $\bfs^{(2)} \in \mathbb{R}_{++}^{d^{(1)}},~\tilde{\bfs}^{(2)} > (1/p-1/r)_+$ for some $r > 0$. Consider the anisotropic Besov space that involves an affine transformation
\begin{equation*}
    \mathcal{F}:= \Set{ f_2 \circ f_1:f_1 = A\cdot + b,~A \in \mathbb{R}^{d^{(1)} \times d^{(0)}},~ b \in \mathbb{R}^{d^{(0)}},~\norm{A}_{\infty}\vee \norm{b}_{\infty} \leq C_a,~ f_2 \in \Besov_{p,q}^{\bfs^{(2)}}}.
\end{equation*}
Then, the approximation error on the function space is bounded as
\begin{equation*}
     \sup_{f_0 \in \mathcal{F}}\inf_{f \in \Phi(L_3,D_3,S_3,B_3)} \norm{f_0 - f}_{L^r} \lesssim N^{-\tilde{\bfs}^{(2)}},  
\end{equation*}
where 
\begin{equation*}
    \begin{aligned}
        L_3 &= L_1(d^{(1)},m^{(2)}, p, r, \bfs^{(2)}, N),\\
        D_3 &= D_1(d^{(1)},m^{(2)}, N),\\
        S_3 &= S_1(d^{(1)}, m^{(2)}, p, r, \bfs^{(2)}, N),\\
        B_3 &= (C_a d^{(1)} + 1)B_1(d^{(1)},m^{(2)},p,r,\bfs^{(2)})
    \end{aligned}
\end{equation*}
for $m^{(2)} \in \mathbb{N}$ satisfying $0 < \tilde{s}^{(2)} < \min\{m^{(2)}, m^{(2)}-1+1/p\}$.
\end{lemma}

To apply \cref{lem:consistency_unknown_var}, we require results concerning the complexity of the model space and the asymptotic properties of neural network models. By the Lipschitz continuity of $\mathrm{clip}$,
$$
\abs{(\mathrm{clip} \circ f_1)(x) - (\mathrm{clip} \circ f_2)(x)} \leq \abs{ f_1(x) - f_2(x)}.
$$
It follows that the covering number of the clipped network class can be bounded by that of the unrestricted class. 
\begin{lemma}[Covering number; Lemma 3 of \citet{suzuki2018adaptivity}]\label{lem:covering}
For any $\epsilon > 0$, $D \ge 3$, and $L \ge 3$.
    \begin{equation*}
        \log \Covering{\epsilon, \tilde\Phi(L, D, S, B), \norm{\cdot}_{\infty}} \leq (S+1) \Log{ 2 \epsilon^{-1} L(B \vee 1 )^{L} (D+1)^{2L}}.
    \end{equation*}
\end{lemma}

To allow for relaxed sparsity with a margin $a > 0$, we define
\begin{equation}\label{eqn:margin_parameter}
\Theta(L, D, S, B, a) = \left\{\theta: (\theta_i I(\abs{\theta_i} > a))_{i=1}^{T} \in \Theta(L, D, S, B)\right\},
\end{equation}
where $T=|\Theta(L, D)|$, and denote the corresponding NN class by $\tilde\Phi(L, D, S, B, a)$ and $\Phi(L, D, S, B, a)$.

\begin{lemma}[Lemma 5 of \citet{lee2022asymptotic}]\label{lem:covering_a}
    For all $\epsilon \ge 2 a L (B \vee 1)^{L-1} (D+1)^{L}$, $D \ge 3$, and $L \ge 3$.
    \begin{equation*}
        \log \Covering{\epsilon, \tilde\Phi(L, D, S, B, a), \norm{\cdot}_{\infty}} \leq (S+1) \Log{2 \epsilon^{-1} L(B \vee 1)^{L} (D+1)^{2L}}.
    \end{equation*}
\end{lemma}

The following lemma, introduced by \citet{schmidt2020nonparametric}, plays a key role in the proofs of \cref{lem:covering} and \cref{lem:covering_a}.
It characterizes the distance between NN models in terms of their parameters.

\begin{lemma}\label{lem:nn_norm}
    For any $\epsilon > 0$ and $\theta,~\theta^\ast \in \Theta(L,D,S,B)$ satisfying $\norm{\theta - \theta^\ast}_{\infty} < \epsilon$,
    \begin{equation*}
        \norm{f_\theta - f_{\theta^\ast}}_{\infty} \leq \epsilon L(B \vee 1)^{L-1} (D+1)^L.
    \end{equation*}
\end{lemma}
\begin{proof}
    See \Cref{subsec:proof_lem:nn_norm}.
\end{proof}

\section{Proofs of the technical results}\label{appendix:proofs}

\subsection{Proof of \Cref{lem:consistency_unknown_var}}\label{subsec:lem:consistency_unknown_var}

\begin{proof}
    We follow the proof of Theorem 2 in \citet{jeong2023art}. Since there exist upper and lower bounds for $\sigma_0^2$ and $\sigma^2$, the rate for $|\sigma^2 - \sigma_0^2|$
    is equivalent to that for $|\sigma - \sigma_0|$; hence we use the latter.
    For every $(f_1,\sigma_1),(f_2,\sigma_2) \in \mathcal{F}\times[\underline \sigma,\overline\sigma]$,
define
    \begin{equation*}
        \rho_L^2((f_1, \sigma_1), (f_2, \sigma_2)) := \norm{f_1 - f_2}_{L^2(P_X)}^2 + \abs{\sigma_1 - \sigma_2}^2.
    \end{equation*}
   By direct calculation, it is straightforward to verify that for some $C_H>1$, 
    \begin{align}
       C_H^{-1}\rho_L((f_1, \sigma_1), (f_2, \sigma_2)) \le  \rho_H(p_{f_1, \sigma_1}, p_{f_2, \sigma_2})  \le C_H \rho_L((f_1, \sigma_1), (f_2, \sigma_2),
       \label{eqn:helleq}
    \end{align}
   where $\rho_H$ denotes the Hellinger distance and $p_{f,\sigma}$ is the density with $f$ and $\sigma$. 
   Therefore, the Hellinger distance is equivalent to $\rho_L$
 up to a constant factor and it suffices to establish the theorem with respect to the Hellinger distance, which admits an exponentially powerful test function.

Using \eqref{eqn:helleq}, the Hellinger entropy can instead be estimated by replacing it with $\rho_L$.
    Note that $\rho_L((f_1, \sigma_1), (f_2, \sigma_2)) \leq \epsilon$ holds if $\norm{f_1 - f_2}_{L^2(P_X)} \leq {\epsilon}/{2}$ and $\abs{\sigma_1 - \sigma_2} \leq {\epsilon}/{2}$. Hence, by \eqref{eqn:lem_con1},
    \begin{align*}
            \log \Covering{\epsilon_n, \mathcal{F}_n\times[\underline \sigma,\overline\sigma], \rho_L} &\lesssim \log \Covering{\epsilon_n/2, \mathcal{F}_n, \norm{\cdot}_{L^2(P_X)}} - \log (\epsilon_n/2) + \log \overline{\sigma} 
            \lesssim n \epsilon_n^2
    \end{align*}
    for all sufficiently large $n$. Using \eqref{eqn:lem_con2}, there exists $C^\prime>0$ such that $\Pi(A_{\epsilon_n}) \ge \Exp{-C^\prime n \epsilon_n^2 }$. Moreover, by Lemma B.2 of \citet{xie2020adaptive},
    \begin{equation*}
        \max \!\left\{ -\mathbb{E}_{f_0, \sigma_0} \!\left[ \log \frac{p_{f,\sigma}^{(n)}}{p_{f_0,\sigma_0}^{(n)}}\right], \mathbb{E}_{f_0,\sigma_0} 
        \!\left[ \left(\log \frac{p_{f,\sigma}^{(n)}}{p_{f_0,\sigma_0}^{(n)}} \right)^2 \right]\right\} \leq C^\dprime \rho_L^2((f,\sigma), (f_0, \sigma_0))
    \end{equation*}
    for some constant $C^\dprime > C^\prime / (C-4)$ and all sufficiently large $n$. We obtain
        \begin{align*}
        &\Pi\!\left(\max \Set{ -\mathbb{E}_{f_0,\sigma_0} \!\!\left[ \log \frac{p_{f,\sigma}^{(n)}}{p_{f_0,\sigma_0}^{(n)}}\right], \mathbb{E}_{f_0,\sigma_0}\!\!\left[ \left(\log \frac{p_{f,\sigma}^{(n)}}{p_{f_0,\sigma_0}^{(n)}} \right)^2 \right] } \leq \epsilon_n^2 \right)
        \\& \ge \Pi\!\left( \rho_L^2((f,\sigma), (f_0, \sigma_0)\} \leq {\epsilon_n^2}/{C^\dprime}\right )\\
        &\ge \Pi(A_{\epsilon_n/\sqrt{C^\dprime}}) \\
        &\ge \Exp{-n\epsilon_n^2 C^\prime / C^\dprime} \\
        &\ge \Exp{-(C-4)n\epsilon_n^2}
    \end{align*}
    for all sufficiently large $n$. In addition, \eqref{eqn:lem_con3} implies
    \begin{equation*}
        \Pi(\mathcal{F} \setminus \mathcal{F}_n) = \int \Pi(\mathcal{F} \setminus \mathcal{F}_n \mid \sigma^2) d\Pi(\sigma^2)  = o(e^{-Cn\epsilon_n^2}).
    \end{equation*}
    We get the desired result using Theorem 2.1 of \citet{ghosal2000convergence}.
\end{proof}

\subsection{Proof of \Cref{lem:consistency_classification}}\label{subsec:proof_lem:consistency_classification}

\begin{proof}
    Let $\rho_L(f_1, f_2) = \norm{(\psi_M \circ f_1) - (\psi_M \circ f_2)}_{L^2(P_X)}$ for $f_1, f_2 \in \mathcal{F}$. The proof proceeds by first establishing the equivalence between the Hellinger distance $\rho_H$ and $\rho_L$, and then deriving the contraction rate with respect to  $\rho_H$. 
    For any $f_1, f_2 \in \mathcal{F}$, let $\eta_1 = \psi_M \circ f_1$ and $\eta_2 = \psi_M \circ f_2$. Then,
        \begin{align*}
            \rho_L^2(f_1, f_2) &= \frac{1}{2}\int \left( \sqrt{\eta_1({x})} - \sqrt{\eta_2({x})} \right)^2 \left( \sqrt{\eta_1({x})} + \sqrt{\eta_2({x})} \right)^2 P_X(d{x})\\
            &\quad + \frac{1}{2}\int \left( \sqrt{1-\eta_1({x})} - \sqrt{1-\eta_2({x})} \right)^2 \left( \sqrt{1-\eta_1({x})} + \sqrt{1-\eta_2({x})} \right)^2 P_X(d{x}).
    \end{align*}
Since both $\eta_1$ and $\eta_2$ lie in $[\delta,1-\delta]$ for some $\delta>0$, there exists a constant $C_H > 0$ such that
        $$C_H^{-1}\rho_L(f_1, f_2) \leq \rho_H(p_{f_1}, p_{f_2}) \leq C_H \rho_L(f_1, f_2). $$
Hence, it suffices to work with the Hellinger distance.  Since the sigmoid function $\psi_M$ is $M/4$-Lipschitz, we obtain
    $$\rho_L(f_1, f_2) = \| \eta_1 - \eta_2 \|_{L^2(P_X)} \leq \frac{M}{4}\| f_1 - f_2 \|_{L^2(P_X)}.$$ Consequently, the Hellinger entropy is bounded by $\log \Covering{\epsilon_n, \mathcal{F}_n, \rho_L} \lesssim \log ({\epsilon_n, \mathcal{F}_n, \norm{\cdot}_{L^2(P_X)}})\lesssim n\epsilon_n^2$. 
Let $p_{f}^{(n)}$ denote the joint density under model \eqref{eqn:model_cls} with $f$. By Lemma 2.8 of \citet{ghosal2017fundamentals}, for any measurable functions $f_1$ and $f_2$, 
    \begin{align*}
        \max \left\{-\mathbb{E}_{f_1}\!\left[\log \frac{p_{f_1}^{(n)}}{p_{f_2}^{(n)}}\right], \mathbb{E}_{f_1}\!\left[\left(\log \frac{p_{f_1}^{(n)}}{p_{f_2}^{(n)}}\right)^2\right]\right\} \lesssim \norm{f_1 - f_2}_{L^2(P_X)}.
    \end{align*}
The remainder of the proof follows the same argument as in \Cref{lem:consistency_unknown_var}.
\end{proof}

\subsection{Proof of \Cref{lem:change_norm}}\label{subsec:proof_change_norm}

\begin{proof}
    We prove the case corresponding to \Cref{thm:reg_ss}; the remaining cases follow by analogous arguments. For \Cref{lem:gyorfi2002}, set $\alpha=1/2$, $\epsilon = M_0 \epsilon_n^2$, $K_1 = K_2 = 4$, and
    \begin{equation*}
    \mathcal{G} = \Set{g(x) = (f(x) - f_0(x))^2: f \in \Phi},
    \end{equation*}
    where $\Phi ={\Phi}(L_{1n}, D_{1n}, S_{1n})$.
    The choice $K_1 = K_2 = 4$ is justified because
    \begin{equation*}
    \abs{g(x)} \leq 4, \quad \mathbb{E}[g(X)^2] \leq 4 \mathbb{E}[g(X)]
    \end{equation*}
for every $g \in \mathcal{G}$.
    Since $n\epsilon_n^2 \rightarrow \infty$, we have
    \begin{equation*}
    \sqrt{n}\alpha \sqrt{1-\alpha}\sqrt{\epsilon_n^2} \ge 288 \max\Set{2K_1, \sqrt{2K_2}}
    \end{equation*}
    for all sufficiently large $n$. Moreover, for any $f_1, f_2 \in \Phi$,
    \begin{equation*}
    \norm{(f_1 - f_0)^2 - (f_2 - f_0)^2}_{1,n} \leq 4 \norm{f_1 - f_2}_{1,n},
    \end{equation*}
    which implies
    \begin{equation*}
    \log \Covering{u, \Set{g \in \mathcal{G}: \frac{1}{n} \sum_{i=1}^n g(x_i)^2 \leq 16t}, \norm{\cdot}_{1, n}} \leq \log \Covering{u/4, \Phi,\norm{\cdot}_{1, n} }.    
    \end{equation*}
Using \cref{lem:covering}, we obtain for every $u\gtrsim \epsilon_n^2$,
    \begin{align*}
 \log \Covering{u/4, \Phi, \norm{\cdot}_{1, n}} 
    &\leq \log \Covering{u/4, \Phi, \norm{\cdot}_{\infty}} \\
    &\leq (S_n+1) \left(\log L_n + L_n \Log{(B_n \vee 1)(D_n + 1)^2} - \log \frac{u}{8}\right) \\
    &\lesssim N_n(\log n)^3 \\
    &\lesssim n\epsilon_n^2.
    \end{align*}
    Therefore, for all $t \ge M_0^2 \epsilon_n^2 / 8$,
    \begin{equation*}
    \begin{aligned}
    \int_{\frac{\alpha(1-\alpha)t}{16\max\Set{K_1, 2K_2}}}^{\sqrt{t}} &\sqrt{\log \Covering{u, \Set{g \in \mathcal{G}: \frac{1}{n} \sum_{i=1}^n g(x_i)^2 \leq 16t}, \norm{\cdot}_{1, n}}} du \lesssim \sqrt{tn\epsilon_n^2}.
    \end{aligned}
    \end{equation*}
    Since $\sqrt{\epsilon_n^2 / t} \leq \sqrt{8 / M_0^2}$ is sufficiently small, we conclude the desired result.
\end{proof}

\subsection{Proof of \Cref{thm:emp_norm}}\label{subsec:proof_reg_empirical}

\begin{proof}
    Let $\mathcal{E}_n$ denote the event
    $$
    \mathcal{E}_n = \left\{ \sup_{f \in \Phi} \frac{\abs{\norm{f - f_0}_{L^2(P_X)}^2 - \norm{f - f_0}_{n}^2}}{M_0 \epsilon_n^2 + \norm{f - f_0}_{L^2(P_X)}^2} \leq \frac{1}{2} \right\}.
    $$
    Then $P_{f_0,\sigma_0}^{(n)}(\mathcal{E}_n^c) \lesssim \exp(-C n \epsilon_n^2)$ for some constant $C > 0$ by \cref{lem:change_norm}.
    On the event $\mathcal{E}_n$, for any $f \in \Phi$, 
    $$
    \norm{f - f_0}_{n} \leq \sqrt{\frac{3}{2}\norm{f - f_0}_{L^2(P_X)}^2 + \frac{M_0}{2} \epsilon_n^2} \leq 2 \norm{f - f_0}_{L^2(P_X)} + \sqrt{\frac{M_0}{2}} \epsilon_n .
    $$
    We now express
    $$
    \begin{aligned}
    \Pi\!\left( 
            \norm{f - f_0}_{n} + \abs{\sigma^2 - \sigma_0^2}  > M_n \epsilon_n \mid \mathcal{D}_n\right) 
    &= \Pi\!\left(\norm{f - f_0}_{n} + \abs{\sigma^2 - \sigma_0^2}  > M_n \epsilon_n \mid \mathcal{D}_n\right) I(\mathcal{E}_n)\\
    &\quad + \Pi\!\left(\norm{f - f_0}_{n} + \abs{\sigma^2 - \sigma_0^2}  > M_n \epsilon_n \mid \mathcal{D}_n\right) I(\mathcal{E}_n^c).
    \end{aligned}
    $$
    For the first term, observe that
    $$
    \begin{aligned}
    &\Pi\!\left(\norm{f - f_0}_{n} + \abs{\sigma^2 - \sigma_0^2}  > M_n \epsilon_n \mid \mathcal{D}_n\right) I(\mathcal{E}_n) \\
    &\leq \Pi\!\left(\norm{f - f_0}_{L^2(P_X)} + \abs{\sigma^2 - \sigma_0^2}  > (M_n - C') \epsilon_n/2 \mid \mathcal{D}_n\right) I(\mathcal{E}_n)
    \end{aligned}
    $$
    for some constant $C' > 0$. By \Cref{thm:reg_ss} (or in \Cref{thm:shrinkage,thm:adaptive_estimation,thm:general_function}), this converges to 0 in $P_{f_0,\sigma_0}^{(n)}$-probability as $n \to \infty$ for any $M_n \to \infty$.
    For the second term, since $P_{f_0,\sigma_0}^{(n)}(\mathcal{E}_n^c) \to 0$, it also converges to 0 in $P_{f_0,\sigma_0}^{(n)}$-probability as $n \to \infty$. This proves the assertion.
\end{proof}

\subsection{Proof of \Cref{lem:anisotropic_approx}}\label{subsec:proof_anisotropic_approx}

\begin{proof}
    Using \Cref{lem:approx_BFbyBS},
      for $K = K(N)$, $K^\ast = \lceil K(1+\nu^{-1}) \rceil$, and $E_N$ with $\abs{E_N} \leq N$, there exists $f_N$ such that
    \begin{equation*}
        \begin{aligned}
        f_N(x) &= \sum_{(k,j) \in E_N} \alpha_{k,j}M^{d,m}_{k,j}(x),\\
        \norm{f - f_N}_{L^r} &\lesssim N^{-\tilde{\bfs}} .
        \end{aligned}
    \end{equation*}
Since for every  $k$ and $j$, 
\begin{align}
\label{eqn:Mdm}
    M_{k,j}^{d,m}(x) = M_{0,0}^{d,m}(2^{\lfloor ks_1^{-1}\rfloor}x_1-j_1,\dots, 2^{\lfloor ks_d^{-1}\rfloor}x_d-j_d),
\end{align}
 \Cref{lem:approx_BSbyNN} implies that there exists an NN approximator $\hat{M}^{d,m}_{k,j}$ satisfying $\|M^{d,m}_{k,j}-\hat{M}^{d,m}_{k,j}\|_\infty\le \epsilon$ for any $\epsilon>0$. By the definition of B-splines, it follows that
    \begin{align*}
        \sum_{(k,j) \in E_N} I\!\left(M^{d,m}_{k,j} (x) \neq 0\right) \leq (m+1)^d (K^\ast + 1).
    \end{align*}
Moreover, using \eqref{eqn:besov_seq_norm}, \eqref{eqn:besov_norm_eqv}, and the inequality $N = 2^{\sum_{j=1}^d \lfloor K \underline{s} / s_j\rfloor} \gtrsim 2^{K d^\ast}$, we obtain
    \begin{align}
    \label{eqn:alphamag}
        \abs{\alpha_{k,j}} \lesssim 2^{K^\ast d^\ast (1/p - \tilde{\bfs})_+} \lesssim N^{(1+ \nu^{-1})(1/p - \tilde{\bfs})_+}.
    \end{align}
We now define the approximator $$\hat{f} = \sum_{(k,j) \in E_N} \alpha_{k,j} \hat{M}_{k,j}^{d,m}.$$ 
Combining the above bounds yields
     \begin{align*}
            |f_N(x) - \hat{f}(x)| 
            &\leq \sum_{(k,j) \in E_N} \abs{\alpha_{k,j}} \abs{M^{d,m}_{k,j} (x) - \hat{M}^{d,m}_{k,j}(x)} \\
            &\leq \epsilon \sum_{(k,j) \in E_N} \abs{\alpha_{k,j}}I\!\left(M^{d,m}_{k,j} (x) \neq 0\right) \\
            &\lesssim \epsilon (m+1)^d (K^\ast+1) 2^{K^\ast d^\ast (1/p - \tilde{\bfs})_+} 
            \\
            &\lesssim \epsilon \log(N) N^{(1+\nu^{-1})(1/p - \tilde{\bfs})_+}\\
            &=N^{-\tilde{\bfs}}.
        \end{align*}
    Using the Lipschitz continuity of $\mathrm{clip}$, we get the desired approximation bound:
    \begin{align*}
        \| f - (\mathrm{clip} \circ \hat{f}) \|_{L^r} 
        &=\|(\mathrm{clip} \circ f) - (\mathrm{clip} \circ \hat{f})\|_{L^r}
        \le \| f - f_N\|_{L^r} + \| f_N - \hat{f}\|_{L^r} 
        \lesssim N^{-\tilde{\bfs}}.
    \end{align*}
Next, we verify that the network parameters of 
$\hat f$ are given as in \eqref{eqn:aniso_hparam}.
Since the identity function $\mathrm{id}(x)=x$ can be represented using the ReLU function as
    $\mathrm{id}(x) = \zeta(x) - \zeta(-x)$,
it follows that for any $L\in\mathbb N$, 
    $$
    \mathrm{id}^{\circ L}:=\underbrace{\mathrm{id} \circ \cdots \circ \mathrm{id}}_{L\text{ times}} = \mathrm{id} \in \tilde{\Phi}(L, 2, 4 L, 1).
    $$
Then, using $\bar{L}_0 = \bar{L}_0(N) = \lceil \log_2 (d \vee m) \rceil \lceil \tilde{\bfs} \log_2 N \rceil$, we can rewrite $\hat f$ as
   \begin{align}
  \label{eqn:fhat}
   \hat{f}(x) =  \sum_{(k,j) \in E_N} \mathrm{id}^{\circ \bar{L}_0}(\alpha_{k,j} \hat{M}_{k,j}^{d,m}(x)),
   \end{align}
   which is also an NN approximator.
   Since $\hat{M}_{k,j}^{d,m}$ requires the same depth, width, and sparsity as $\hat M$ in  \Cref{lem:approx_BSbyNN}, each summand $\mathrm{id}^{\circ \bar{L}_0}(\alpha_{k,j} \hat{M}_{k,j}^{d,m}(x))$ has depth $L_0 + \bar{L}_0$, width $D_0$, and sparsity $L_0D_0^2+4 \bar L_0 +1$.
Therefore, $\hat f$ expressed as in \eqref{eqn:fhat} requires depth $L_0 + \bar{L}_0 \leq L_1$, width $ND_0 $, and sparsity $(L_0D_0^2+4\bar L_0+1)N\le S_1$.
To evaluate the magnitude of the network parameters, note from \eqref{eqn:Mdm} that the magnitudes of the parameters in the first layer of $\hat{M}_{k,j}^{d,m}$ are bounded by $2^{K^\ast \tilde{\bfs}}  \leq C_{11} N^{(1 + \nu^{-1})/\underline{s}}$ for some constant $C_{11}>0$. All remaining layers have the same magnitude as those of $\hat M$, namely $B_0$.
In addition, by \eqref{eqn:alphamag}, the magnitudes of the parameters in the first layer of $t\mapsto \mathrm{id}^{\circ \bar{L}_0}(\alpha_{k,j} t)$ are bounded by $C_{12} N^{(1+ \nu^{-1})(1/p - \tilde{\bfs})_+}$ for some constant $C_{12} > 0$.
    Let $B_{11}(N) = C_{11} N^{(1+\nu^{-1})/\underline{s}}B_0^{L_0}$ and $B_{12}(N) = C_{12} N^{(1+\nu^{-1})(1/p - \tilde{s})_+}B_0^{\bar{L}_0}$, and define
    \begin{align*}
    c_{1l}&=
        \begin{cases}
            B_{11}(N)^{(L_0+1)^{-1}} / (C_{11} N^{(1+\nu^{-1})/\underline{s}}), &l=1, \\
            B_{11}(N)^{(L_0+ 1)^{-1}} / B_0 & l=2,\dots, L_0+1,
            \end{cases} \\
    c_{2l}&=
        \begin{cases}
            B_{12}(N)^{(\bar{L}_0 + 1)^{-1}} / (C_{12} N^{(1+\nu^{-1})(1/p - \tilde{s})_+}), &l=1, \\
            B_{12}(N)^{(\bar{L}_0 + 1)^{-1}}/ B_0 & l=2,\dots, \bar L_0+1.
            \end{cases}
    \end{align*}
It is easy to verify that $\prod_{l^\prime = 1}^{l} c_{1l^\prime} \leq 1$ and $\prod_{l^\prime = 1}^{l} c_{2l^\prime} \leq 1$ for any $l$ and sufficiently large $N$. 
Therefore, by \cref{lem:re_scaling}, the rescaled $\hat{h}$ has parameter magnitudes bounded by $B_{11}(N)^{(L_0+1)^{-1}} \vee B_{12}(N)^{(L_0+1)^{-1}}$.
Since $L_0(d,m,N) \ge \bar{L}_0 = \lceil \log_2 (d \vee m) \rceil \lceil \tilde{\bfs} \log_2 N \rceil$ for sufficiently large $N$, we obtain
    \begin{align*}
            B_{11}(N)^{(L_0+1)^{-1}} \vee B_{12}(N)^{(\bar{L}_0+1)^{-1}} 
            &\leq B_0 \left((C_{11} \vee C_{12}) N^{(1+ \nu^{-1})[(1/\underline{s}) \vee (1/p - \tilde{\bfs})_+]}\right)^{(\bar{L}_0+1)^{-1}} \\ 
            &\leq B_0 N^{(1+ \nu^{-1})[(1/\underline{s}) \vee (1/p - \tilde{\bfs})_+]\bar{L}_0^{-1}} \\
            &\leq B_0 \exp\left({\frac{(1+ \nu^{-1})[(1/\underline{s})\vee(1/p - \tilde{\bfs})_+]\log 2}{\lceil \tilde{\bfs} \log_2 (d \vee m) \rceil}}\right) \\
            &= B_1,
    \end{align*}
    which concludes $\hat f \in\Phi(L_1,D_1,S_1,B_1)$.
\end{proof}

\subsection{Proof of \Cref{lem:composite_anisotropic_approx}}\label{subsec:proof_lem:composite_anisotropic}

\begin{proof}
    Our proof is similar to the proof of Theorem 1 in \citet{suzuki_deep_2021}.
    Let $f_0 = f_{0,H} \circ f_{0,H-1} \cdots\circ f_{0,1}$. Since $\tilde{\bfs}^{(h)} > 1/p$ for $h \ge 2$, by \cref{lem:anisotropic_approx}, for each $f_{0,h,k}$ there exists
    \begin{align*}
    f_{h,k} \in \Phi(&L_1(t^{(h)}, m^{(h)}, p,\infty, \bfs^{(h)}, N), D_1(t^{(h)}, m^{(h)}, N), \\
                    & S_1(t^{(h)}, m^{(h)},p, \infty, \bfs^{(h)}, N), B_1(t^{(h)},m^{(h)},p, \infty, \bfs^{(h)}))
    \end{align*}
    such that $\norm{f_{0,h,k} - f_{h,k}}_{\infty} \lesssim N^{-\tilde{\bfs}^{(h)}}$.
    For $h = 1$, each $f_{0,1,k}$ there exists
    \begin{align*}
    f_{1,k} \in \Phi(&L_1(d^{(1)}, m^{(1)}, p, r, \bfs^{(1)}, N), D_1(d^{(1)}, m^{(1)}, N), \\
                    & S_1(d^{(1)}, m^{(1)}, p, r, \bfs^{(1)}, N), B_1(d^{(1)},m^{(1)},p, r, \bfs^{(1)}))
    \end{align*}
    such that $\norm{f_{0,1,k} - f_{1,k}}_{L^r} \lesssim N^{-\tilde{\bfs}^{(1)}}$. 
   To ensure that each block produces outputs in $[0,1]$, we apply the clipping function $\mathrm{clip}_{[0,1]}$, as described in \cref{rmk:clipping}, which adds one layer to each block. 
    Let $f = f_{H} \circ f_{H-1} \cdots\circ f_{1} \in \Phi(L_2, D_2, S_2, B_2)$ as in \citet{schmidt2020nonparametric}.
    Then, 
    \begin{align*}
        &\norm{f_0 - f}_{L^r} \\
        &= \norm{f_{0,H} \circ f_{0,H-1} \cdots\circ f_{0,1} - f_{H} \circ f_{H-1} \cdots\circ f_{1}}_{L^r} \\
        &\leq \sum_{h=1}^H \norm{f_{0,H} \circ f_{0,H-1} \cdots\circ f_{0,h} \circ f_{h-1} \cdots f_{1} - f_{0,H} \circ f_{0,H-1}  \cdots \circ f_{0,h+1} \circ f_{h} \cdots f_{1}}_{L^r} \\
        &\leq \sum_{h=1}^H \norm{f_{0,H} \circ f_{0,H-1} \cdots\circ f_{0,h}  - f_{0,H} \circ f_{0,H-1}  \cdots \circ f_{0,h+1} \circ f_{h}}_{L^{r_h}},
    \end{align*}
    where $r_1=r$ and $r_h = \infty$ for $h \ge 2$. By \cref{rmk:besov-embeddings},
    $$
    f_{0,h,k} \in C^{\underline{\bfs}^{(h)} [ 1- 1/(p\tilde{\bfs}^{(h)})] \wedge 1} = C^{(\underline{\bfs}^{(h)} - t^{\ast(h)}/p) \wedge 1} = C^{\gamma_h^\prime},\quad h \ge 2,
    $$
     where $\gamma_h^\prime := (\underline{\bfs}^{(h)} - t^{\ast(h)}/p) \wedge 1$. Therefore,
    \begin{equation*}
    \begin{aligned}
    \norm{f_{0,H} \circ f_{0,H-1} \cdots\circ f_{0,h} - f_{0,H} \circ f_{0,H-1} \cdots \circ f_{0,h+1} \circ f_{h}}_{\infty} 
    &\lesssim \norm{f_{0, h} - f_h}_{\infty}^{\prod_{h^\prime = h+1}^H \gamma_{h^\prime}^\prime} \\
    &\lesssim N^{-\tilde{\bfs}^{(h)} \prod_{h^\prime = h+1}^H \gamma_{h^\prime}^\prime} \\
    &= N^{-\tilde{\bfs}^{\ast(h)}}
    \end{aligned}
    \end{equation*}
    for $h \ge 2$. For $h = 1$,
    \begin{equation*}
    \begin{aligned}
        \norm{f_{0,H} \circ f_{0,H-1} \cdots\circ f_{0,1} - f_{0,H} \circ f_{0,H-1} \cdots \circ f_{1}}_{L^r} 
        &\lesssim \norm{f_{0,1} - f_1}_{L^r}^{\prod_{h^\prime = 2}^H \gamma_{h^\prime}^\prime} \\
        &\lesssim N^{-\tilde{\bfs}^{\ast(1)}}.
    \end{aligned}
    \end{equation*}
    Therefore,
    $
        \norm{f_0 - f}_{L^r} \lesssim \max_{h \in [H]} N^{-\tilde{\bfs}^{\ast(h)}} = N^{-\tilde{\bfs}^{\ast}}
    $.
\end{proof}

\subsection{Proof of \Cref{lem:nn_norm}}\label{subsec:proof_lem:nn_norm}

\begin{proof}
For $f \in \Phi(L, D, S, B)$ expressed as
    $f({x}) = (\NNW^{(L)}\zeta(\cdot) + \NNB^{(L)}) \circ \cdots \circ (\NNW^{(2)}{\zeta(\cdot)} + \NNB^{(2)}) \circ (\NNW^{(1)}{x} + \NNB^{(1)})$,
    define
    \begin{equation*}
        \begin{aligned}
            \mathcal{A}_k^+(f)({x}) &= \zeta \circ (\NNW^{(k-1)} \zeta(\cdot) + \NNB^{(k-1)}) \circ \cdots \circ (\NNW^{(2)} \zeta(\cdot) + \NNB^{(2)}) \circ (\NNW^{(1)}{x} + \NNB^{(1)}), \\
            \mathcal{A}_k^-(f)({x}) &= (\NNW^{(L)}\zeta(\cdot) + \NNB^{(L)}) \circ \cdots \circ (\NNW^{(k+1)} \zeta(\cdot) + \NNB^{(k+1)}) \circ (\NNW^{(k)}x + \NNB^{(k)}),
        \end{aligned}
    \end{equation*}
    for $k=2,\cdots, L$, and let $\mathcal{A}_{L+1}^-(f)({x}) = \mathcal{A}_1^+(f)({x}) = {x}$. Then,
    $$f({x}) = \mathcal{A}_{k+1}^-(f) \circ (\NNW^{(k)}\cdot + \NNB^{(k)}) \circ \mathcal{A}_k^+(f)({x}).$$
    Using the definition of $\Phi(L, D, S, B)$,
    \begin{equation*}
        \begin{aligned}
            \norm{\mathcal{A}_k^+(f)({x})}_{\infty} 
            &\leq \max_j \norm{\NNW_{j,:}^{(k-1)}}_1 \norm{\mathcal{A}_{k-1}^+(f)({x})}_{\infty}  + \|\NNB^{(k-1)}\|_{\infty} \\
            &\leq DB\norm{\mathcal{A}_{k-1}^+(f)({x})}_{\infty} + B \\
            &\leq (D+1)(B \vee 1) \norm{\mathcal{A}_{k-1}^+(f)({x})}_{\infty}  \\
            &\leq (D+1)^{k-1}(B \vee 1)^{k-1},
        \end{aligned}
    \end{equation*}
    where $A_{j, :}$ denotes the $j$-th row of matrix $A$.
    Similarly,
    \begin{equation*}
        \abs{\mathcal{A}_k^-(f)({x}_1) - \mathcal{A}_k^-(f)({x}_2)} \leq (BD)^{L-k+1} \norm{{x}_1 - {x}_2}_{\infty}.
    \end{equation*}
    Fix $\epsilon > 0$ and $\theta \in \Theta(L,D,S,B)$. For any $\theta^\ast \in \Theta(L,D,S,B)$ satisfying $\norm{\theta - \theta^\ast}_{\infty} < \epsilon$, we obtain
        \begin{align*}
            &\abs{f_\theta({x}) - f_{\theta^\ast}({x})} \\
            &= \bigg\vert \sum_{k=1}^L \mathcal{A}_{k+1}^-(f_{\theta^\ast}) \circ (\NNW^{(k)} \cdot + \NNB^{(k)}) \circ \mathcal{A}_k^+(f_\theta)({x}) - \mathcal{A}_{k+1}^-(f_{\theta^\ast}) \circ (\NNW^{(k)^\ast}\! \cdot + \NNB^{(k)^\ast}) \circ \mathcal{A}_k^+(f_\theta)({x}) \bigg\vert \\
            &\leq \sum_{k=1}^L (BD)^{L-k} \bigg\Vert (\NNW^{(k)} \cdot + \NNB^{(k)}) \circ \mathcal{A}_k^+(f_\theta)({x}) - (\NNW^{(k)^\ast}\! \cdot + \NNB^{(k)^\ast}) \circ \mathcal{A}_k^+(f_\theta)({x})\bigg\Vert_{\infty} \\
            &\leq \sum_{k=1}^L (BD)^{L-k} \epsilon \left[D(B \vee 1)^{k-1}(D+1)^{k-1} +1 \right] \\
            &\leq \sum_{k=1}^L (BD)^{L-k} \epsilon (B \vee 1)^{k-1}(D+1)^k \\
            &\leq \epsilon L(B \vee 1)^{L-1} (D+1)^L.
        \end{align*}
    This proves the assertion.
\end{proof}

\subsection{Proof of \Cref{thm:reg_ss}}\label{subsec:proof_reg_ss}

\begin{proof}
    Let $\mathcal{F} = {\Phi}(L_{1n},D_{1n},S_{1n})$.
    By \cref{lem:consistency_unknown_var}, it suffices to show that there exists a subset $\mathcal{F}_n \subset \mathcal{F}$ such that
    \begin{enumerate}[leftmargin=1.8em,itemsep=-0.3em,topsep=0em,label=(\alph*)]
        \item $\log \Covering{\epsilon_n, \mathcal{F}_n, \norm{\cdot}_{L^2(P_X)}} \lesssim n\epsilon_n^2$
        \item $-\log \Pi(A_{\epsilon_n}) \lesssim n \epsilon_n^2$
        \item $\sup_{\sigma \in [\underline{\sigma}, \overline{\sigma}]} \Pi(\mathcal{F} \setminus \mathcal{F}_n \mid \sigma) = o\!\left(e^{-Cn\epsilon_n^2}\right)$
    \end{enumerate}
    for a sufficiently large $C$. Let $\mathcal{F}_n = {\Phi}(L_{1n}, D_{1n}, S_{1n}, {B_{n}})$, where $B_n = n$. To verify (c), observe that
    \begin{align*}
            \Pi(\mathcal{F} \setminus \mathcal{F}_n \mid \sigma) = \Pi(\mathcal{F} \setminus \mathcal{F}_n) = \Pi \!\left( \exists j: \abs{\theta_j} > B_{n} \mid L_{1n},D_{1n},S_{1n} \right) = 1 - (1-v_n)^{S_{1n}},
    \end{align*}
    where $v_n = \int
    _{|u|>B_n}\tilde{\pi}_{SL}(u)du$.
The direct calculation of the number of network parameters yields
    $$T= \abs{\Theta(L,D)} = D^2(L-1)+D(d+L+1)+1,$$
which implies
\begin{align}
\label{eqn:T}
T_{1n}:=\abs{\Theta(L_{1n},D_{1n})}\asymp L_{1n}D_{1n}^2 .
\end{align}
Moreover, $v_n \leq e^{-k_1 B_n}$ for some $k_1 > 0$ by Assumption \aref{cond:ss_support}.
Therefore, by Bernoulli's inequality, we obtain
$$
1 - (1-v_n)^{S_{1n}}\le S_{1n} v_n \le T_{1n} v_n \le  e^{-k_1 B_n + \log T_{1n}} = o\!\left(e^{-Cn\epsilon_n^2}\right),
$$
for a sufficiently large $C$.
    By \cref{lem:covering},
    \begin{align*}
             \log \Covering{ \epsilon_n, \mathcal{F}_n, \norm{\cdot}_{L^2(P_X)}} &\leq \log \Covering{\epsilon_n, \mathcal{F}_n, \norm{\cdot}_{\infty}} \\
            &\leq (S_{1n}+1) \bigg[ \log L_{1n} + L_{1n} \Log{ ({B_{n}} \vee 1) (D_{1n} + 1)^2}  - \log \frac{\epsilon_n}{2} \bigg] \\
            &\lesssim N_n (\log n)^3 \\
            &\lesssim n\epsilon_n^2
    \end{align*}
    for sufficiently large $n$. Thus, (a) holds.
    Given \cref{lem:anisotropic_approx}, \cref{rmk:hyper_bigo}, and Assumption~\aref{assum:a-1}, there exists 
    $\hat{f}_n := f_{\hat{\theta}} \in \Phi(L_{1n}, D_{1n}, S_{1n}, B_1) \subset 
    \mathcal{F}_n$ for $\hat{\theta}\in \Theta(L_{1n}, D_{1n}, S_{1n}, B_1) $
    such that
    \begin{equation}\label{eqn:mle_l2}
        \lVert\hat{f}_n - f_0\rVert_{L^2(P_X)}
        \le C_1 N^{-\tilde{s}} \le\epsilon_n/4,
    \end{equation}
    for some $C_1>0$ and all sufficiently large $n$.
     Let $\hat{\gamma}$ denote the index set of nonzero components in $\hat{\theta}$, and let $\hat{\theta}_{\hat{\gamma}}$ denote the corresponding nonzero values. We define $\Theta(\hat{\gamma}; L_{1n}, D_{1n}, S_{1n}, B_1)\subset\Theta(L_{1n}, D_{1n}, S_{1n}, B_1)$ as the subset of parameter space in which only the components indexed by $\hat{\gamma}$ are nonzero. The corresponding NN space is denoted by
$$\tilde{\mathcal{F}}_n(\hat{\gamma}) = \Phi(\hat{\gamma}; L_{1n}, D_{1n}, S_{1n}, B_1).$$
    Using Assumption \aref{assum:a-3}, we have
    \begin{equation*}
    \log \Pi(\abs{\sigma - \sigma_0} \leq \epsilon_n/2) \gtrsim \log \epsilon_n \gtrsim -\log n
    \end{equation*}
    and hence
    \begin{equation*}
        \Pi(A_{\epsilon_n}) \gtrsim \Pi(f \in \mathcal{F}_n: \norm{f - f_0}_{L^2(P_X)} \leq \epsilon_n/2) - \log n.
    \end{equation*}
    Given \eqref{eqn:mle_l2}, there exists a constant $C_1>0$ such that
    \begin{equation*}
        \begin{aligned}
            \Pi\!\left(f \in \mathcal{F}_n: \norm{f - f_0}_{L^2(P_X)} \leq \epsilon_n/2\right) &\ge \Pi\!\left(f \in \mathcal{F}_n: \lVert f - \hat{f}_n\rVert_{L^2(P_X)} \leq \epsilon_n/4\right) \\
            &\ge \Pi\!\left(f \in \mathcal{F}_n: \lVert f - \hat{f}_n\rVert_{\infty} \leq \epsilon_n/4\right) \\
            &\ge \Pi\!\left(f \in {\tilde{\mathcal{F}}_n(\hat{\gamma})}: \lVert f - \hat{f}_n\rVert_{\infty} \leq \epsilon_n/4\right).
        \end{aligned}
    \end{equation*}
    Using \cref{lem:nn_norm},
        \begin{align*}
             &\Pi\!\left( f \in {\tilde{\mathcal{F}}_n(\hat{\gamma})}: \| f - \hat{f}_n\|_{\infty} \leq \epsilon_n/4\right) \\
             &\ge \Pi\bigg(\theta \in \mathbb{R}^{T_{1n}}: \theta_{\hat{\gamma}^c} = 0, \|\theta_{\hat{\gamma}}\|_{\infty} \leq B_{1}, \|\hat{\theta}_{\hat{\gamma}} - \theta_{\hat{\gamma}}\|_{\infty} \leq \frac{\epsilon_n}{ 4(D_{1n}+1)^{L_{1n}} L_{1n} (B_{1} \vee 1)^{L_{1n}-1}} \bigg) \\
             &\ge{\left( t_n \inf_{u \in [-B_1,B_1]} \tilde{\pi}_{SL}(u) \right)^{S_{1n}} \binom{T_{1n}}{S_{1n}}^{-1}} \\
             &ge{\left( t_n (D_{1n} + 1)^{-L_{1n}} \inf_{u \in [-B_1,B_1]} \tilde{\pi}_{SL}(u) \right)^{S_{1n}}}
        \end{align*}
    where $t_n = {\epsilon_n}/{[2(D_{1n}+1)^{L_{1n}} L_{1n} (B_{1} \vee 1)^{L_{1n}-1}]}$.
    The last inequality follows from
    \begin{align*}
        \binom{T_{1n}}{S_{1n}} = \frac{T_{1n}(T_{1n}-1) \cdots (T_{1n}-S_{1n}+1)}{S_{1n}!} \leq T_{1n}^{S_{1n}} \leq (D_{1n}+1)^{L_{1n}S_{1n}}.
    \end{align*}
    Using
    \begin{equation*}
        -\log t_n \leq \Log{\frac{ (D_{1n}+1)^{L_{1n}} L_{1n}  (B_{1} \vee 1)^{L_{1n}-1}}{2\epsilon_n}} \lesssim (\log n)^2
    \end{equation*}
    and Assumption \aref{cond:ss_tail}, we obtain
    \begin{equation*}
        \begin{aligned}
            -\log \Pi(A_{\epsilon_n}) & \lesssim -S_{1n} \log \!\left( t_n (D_{1n}+1)^{-L_{1n}} \inf_{u \in [-B_1, B_1]} \tilde{\pi}_{SL}(u) \right) + \log n \\
            &\lesssim S_{1n} (\log n)^2 \\
            &\lesssim N_n (\log n)^3 \\
            &\lesssim n \epsilon_n^2,
        \end{aligned}
    \end{equation*}
    which verifies (b).
\end{proof}

\subsection{Verification of \Cref{rmk:unbounded_ss}}\label{appendix:sparse_ubd}

Note that \citet{kong2024posterior} established posterior consistency over unbounded parameter spaces.
Their analysis allows for greater flexibility in prior specification by leveraging a complexity bound, and the result builds upon the earlier work of \citet{kohler2021rate}.
Motivated by their approach, we show that posterior concentration at the rate $n\epsilon_n^2 = N_n(\log n)^4$ remains valid (in contrast to the original rate $n\epsilon_n^2 = N_n(\log n)^3$) by employing the covering number bound stated below.

\begin{lemma}[Covering number for sparse unbounded NNs]\label{lem:covering_unbounded}
For all $\epsilon > 0$, the following holds:
$$
\log \Covering{\epsilon, \Phi(L, D, S), \norm{\cdot}_{L^p}} \leq C_{V_1} LS \log S \log\left( \frac{K}{\epsilon^p} \right) + \log\left(C_{V_2} LS \log S\right),
$$
for some positive constants $C_{V_1} > 0,~ C_{V_2} > 0$ and $K > 0$.
\end{lemma}

\begin{proof}
By Theorem 7 of \citet{bartlett2019nearly}, the VC-dimension $V_{\Phi(L, D, S)}^+$ of the class $\Phi(L, D, S)$ satisfies $V^+_{\Phi(L, D, S)} \leq C_V LS \log S$ for some constant $C_V > 0$. Then, by Theorem 2.6.4 of \citet{vaart1997weak}, it follows that
$$
\log \Covering{\epsilon, \Phi(L, D, S), \norm{\cdot}_{L^p}} \leq \log \!\left( C' V^+_{\Phi(L, D, S)}  (K/\epsilon^p)^{V^+_{\Phi(L, D, S)}} \right),
$$
for some positive constants $C' > 0$ and $K > 0$.
Therefore, the stated result follows.
\end{proof}

Now we are ready to prove the assertion in \cref{rmk:unbounded_ss}.
\begin{proof}
Let $\mathcal{F} = \Phi(L_{1n},D_{1n},S_{1n})$.  
By \Cref{lem:consistency_unknown_var}, it suffices to construct a set $\mathcal{F}_n \subset \mathcal{F}$ such that the following conditions hold:
\begin{enumerate}[leftmargin=1.8em,itemsep=-0.3em,topsep=0em,label=(\alph*)]
    \item $\log \Covering{\epsilon_n, \mathcal{F}_n, \norm{\cdot}_{L^2(P_X)}} \lesssim n\epsilon_n^2$,
    \item $-\log \Pi(A_{\epsilon_n}) \lesssim n\epsilon_n^2$,
    \item $\sup_{\sigma \in [\underline{\sigma}, \overline{\sigma}]} \Pi(\mathcal{F} \setminus \mathcal{F}_n \mid \sigma) = o\left(e^{-Cn\epsilon_n^2}\right)$,
\end{enumerate}
for some constant $C > 0$. Let $\mathcal{F}_n = {\Phi}(L_{1n}, D_{1n}, S_{1n})$. Note that condition (c) is trivially satisfied. Using \Cref{lem:covering_unbounded} and Assumption~\aref{assum:a-2}, we obtain
\begin{align*}
    \log \Covering{\epsilon_n, \mathcal{F}_n, \norm{\cdot}_{L^2(P_X)}} 
    &\leq \log \Covering{\epsilon_n/R, \mathcal{F}_n, \norm{\cdot}_{L^2}} \\
    &\lesssim L_{1n} S_{1n} \log S_{1n} \log \epsilon_n^{-1} \\
    &\lesssim N_n (\log n)^4 \\
    &\lesssim n\epsilon_n^2,
\end{align*}
where the last line follows since $L_{1n} \asymp \log n$ and $S_{1n} \asymp N_n \log n$. The remaining parts follow identically from the proof of \Cref{thm:reg_ss}.
\end{proof}

\subsection{Verification of \Cref{ex:relaxed_ss}}\label{subsec:proof_ex_relaxed_ss}

 Assumption~\aref{cond:shr_tail} holds trivially. To verify Assumption~\aref{cond:shr_support}, observe that
\begin{equation*}
    \int_{[-K_n, K_n]^c} \tilde{\pi}_{SH}(u)\, du 
    \leq \pi_{1n} \Pr(|Z| > K_n/\sigma_{1n}) 
    \leq C_1 \exp\left(- C_2 K_n^{1/k} / \sigma_{1n}^{1/k}\right),
\end{equation*}
for all sufficiently large $n$, where $Z$ is a random variable having density $\varphi_k$.
Since $\sigma_{1n}$ is decreasing,
$$
\log\left( \int_{[-K_n, K_n]^c} \tilde{\pi}_{SH}(u)\, du \right) \lesssim -  K_n^{1/k} / \sigma_{1n}^{1/k} \lesssim -K_n^{1/k} \lesssim -K_n,
$$
for $k \leq 1$. This verifies Assumption~\aref{cond:shr_support}. 
Next, observe that
    \begin{align*}
    \int_{[-a_n, a_n]^c} \tilde{\pi}_{SH}(u)\, du &\leq \pi_{2n} \left(1 - \frac{a_n}{C_u}\right) + \pi_{1n} \Pr(|Z| > (2C_A (\log n)^2)^{k}) \\
    &\leq 2e^{-2C_A(\log n)^2} \\
    &\leq e^{-C_A(\log n)^2}.
\end{align*}
Hence, Assumption~\aref{cond:shr_spike} is satisfied.

\subsection{Verification of \Cref{ex:gaussian_mixture}}\label{subsec:proof_ex_gaussian_mixture}

Assumptions~\aref{cond:shr_support} and \aref{cond:shr_spike} can be verified in the same manner as in \cref{ex:relaxed_ss}.
It therefore suffices to verify Assumption~\aref{cond:shr_tail}. Note that $ \tilde{\pi}_{SH}(u) \ge \pi_{2n}{\sigma_2^{-1}}\phi(u/\sigma_2)$, where $\phi$ denotes the density of the standard normal distribution and $\sigma_2$ is a positive standardized deviation. Therefore,
\begin{equation*}
    -\inf_{u \in [-B_1, B_1]} \log  \tilde{\pi}_{SH}(u) \leq - \log \pi_{2n} + \frac{1}{2} \log (2\pi \sigma_2^2) + \frac{B_1^2}{2\sigma_2^2} \lesssim (\log n)^2.
\end{equation*}

\subsection{Proof of \Cref{thm:shrinkage}}\label{subsec:proof_shrinkage}

\begin{proof}
    Let $\mathcal{F} = {\Phi}(L_{1n},D_{1n})$. By \cref{lem:consistency_unknown_var}, it suffices to show that there exists a subset $\mathcal{F}_n \subset \mathcal{F}$ such that
    \begin{enumerate}[leftmargin=1.8em,itemsep=-0.3em,topsep=0em,label=(\alph*)]
        \item $\log \Covering{\epsilon_n, \mathcal{F}_n, \norm{\cdot}_{L^2(P_X)}} \lesssim n\epsilon_n^2$
        \item $-\log \Pi(A_{\epsilon_n}) \lesssim n \epsilon_n^2$
        \item $\sup_{\sigma \in [\underline{\sigma}, \overline{\sigma}]} \Pi(\mathcal{F} \setminus \mathcal{F}_n \mid \sigma ) = o\!\left(e^{-C^\prime n \epsilon_n^2}\right)$
    \end{enumerate}
    for a sufficiently large $C^\prime$. Let $\mathcal{F}_n = {\Phi}(L_{1n}, D_{1n}, S_{1n}, B_{n}, a_n)$ as defined in \eqref{eqn:margin_parameter}, where $B_n = n$.
    Since $B_n^{-(L_{1n}-1)} \ge \exp(-L_{1n} \log n)$ and $\epsilon_{n}/[2L_{1n}(D_{1n}+1)^{L_{1n}}] \ge \exp(-L_{1n} \log n)$, we obtain 
    \begin{align}
    \label{eqn:anbound}
    \epsilon_n / [2L_{1n} (B_{n} \vee 1)^{L_{1n}-1} (D_{1n}+1)^{L_{1n}} ]
    &\ge \exp(-2 L_{1n} \log n) 
    = a_n 
    \end{align}
for all sufficiently large $n$. Hence, it is easy to show that (a) holds using \cref{lem:covering_a} as in the proof of \Cref{thm:reg_ss}.
    Let $v_n = \int_{|u|>B_n} \tilde{\pi}_{SH} \left( u \right) du$ and $u_n = \int_{|u|> a_n} \tilde{\pi}_{SH}(u) du$. Then,
    \begin{align*}
            \Pi(\mathcal{F} \setminus \mathcal{F}_n \mid \sigma) & = \Pi(\mathcal{F} \setminus \mathcal{F}_n ) \\
            &\leq \Pi\! \left( \exists j: \abs{\theta_j} > B_{n} \mid  L_{1n},D_{1n} \right)  + \Pi\!\left(\sum_{j=1}^{T_{1n}} I(\abs{\theta_j} >a_n) > S_{1n}\mid  L_{1n},D_{1n} \right) \\
            &= (1 - (1-v_n)^{T_{1n}}) + \Pr ( S > S_{1n} ),
    \end{align*}
    where $S \sim B(T_{1n}, u_n)$ is a binomially distributed random variable.
    Since $v_n \leq e^{-k_1 B_n}$ for some $k_1 > 0$ by Assumption~\aref{cond:shr_support}, we have
    \begin{equation*}
        \begin{aligned}
        1 - (1-v_n)^{T_{1n}} &\leq T_{1n} v_n \le e^{-k_1 B_n + \log T_{1n}} = o\!\left(e^{-K_2n\epsilon_n^2}\right)
        \end{aligned}
    \end{equation*}
    for a sufficiently large $K_2>0$.
    Let $r_n = {S_{1n}}/{(T_{1n}u_n)}$.
    Using the multiplicative Chernoff bound of binomial distributions,
    \begin{align*}
            \Pr ( S > S_{1n} ) &\leq \Exp{ - {T_{1n} u_n}[  r_n \log r_n -(r_n-1) ] } 
            \le \Exp{ -C_1 S_{1n} \log r_n }
    \end{align*}
    for some $C_1>0$. By \eqref{eqn:T} together with Assumption~\ref{cond:shr_spike}, we have $\log r_n  \ge K_3 (\log n)^2$ for a sufficiently large constant $K_3 > 0$.
    Therefore,
\begin{align*}
       \Pr(S > S_{1n}) \leq  \exp(-K_3 n\epsilon_n^2).
\end{align*}
    Next, as in the proof of \Cref{thm:reg_ss}, there is a constant $C_1>0$ and $f_{\hat{\theta}} \in {\Phi(L_{1n},D_{1n},S_{1n},B_1)}$ for $\hat{\theta} \in {\Theta(L_{1n},D_{1n},S_{1n},B_1)}$  such that
    \begin{equation*}
        \norm{{f_{\hat{\theta}}} - f_0}_{L^2(P_X)} \leq C_1 N_n^{-\tilde{s}} \leq \epsilon_n/8
    \end{equation*}
    for all sufficiently large $n$. Let $\hat{\gamma}$ denote the index set of nonzero components in $\hat{\theta}$. 
Define the subset of the parameter space $\Theta(\hat{\gamma}; L_{1n},D_{1n}, S_{1n}, B_1, a_n)\subset \Theta(L_{1n},D_{1n}, S_{1n}, B_1, a_n)$
    whose magnitudes exceed $a_n$ only at $\hat\gamma$, and let $\tilde{\mathcal{F}}_n(\hat{\gamma}) = {\Phi}(\hat{\gamma}; L_{1n},D_{1n}, S_{1n},B_1, a_n)$ be the corresponding NN space. Then, by \eqref{eqn:anbound} and \cref{lem:nn_norm}, there exists $\hat{\theta}(\hat{\gamma}) \in \Theta(\hat{\gamma}, L_{1n},D_{1n},S_{1n}, B_1, a_n)$ such that $\|{f_{\hat{\theta}}} - f_{\hat{\theta}(\hat{\gamma})}\|_{L^2(P_X)} \leq \epsilon_n / 8$.
    Let $\hat{f}_n = f_{\hat{\theta}(\hat{\gamma})}$. For all sufficiently large $n$, we obtain
    \begin{align*}
            \Pi\!\left(f \in \mathcal{F}_n: \norm{f - f_0}_{L^2(P_X)} \leq \epsilon_n/2\right) 
            &\ge \Pi\!\left(f \in \mathcal{F}_n: \|f - \hat{f}_n\|_{L^2(P_X)} \leq \epsilon_n/4\right)\\
            &\ge\Pi\!\left(f \in \tilde{\mathcal{F}}_n(\hat{\gamma}): \|f - \hat{f}_n\|_{\infty} \leq \epsilon_n/4\right).
    \end{align*}
By \Cref{lem:nn_norm},
    \begin{align*}
                &\Pi\!\left( f \in  \tilde{\mathcal{F}}_n(\hat{\gamma}): \| f - \hat{f}_n\|_{\infty} \leq \epsilon_n/4\right)\\
                &\ge \Pi\!\left(\theta_{\hat{\gamma}^c} \in [-a_n, a_n]^{T_{1n}-S_{1n}}, \norm{\theta_{\hat{\gamma}}}_{\infty} \leq B_{1}, \|\hat{\theta}_{\hat{\gamma}} - \theta_{\hat{\gamma}}\|_{\infty} \!\leq \frac{\epsilon_n/4}{(D_{1n}+1)^{L_{1n}} L_{1n} (B_{1} \vee 1)^{L_{1n}-1}} \right) \\
                &\ge (1-u_n)^{T_{1n} - S_{1n}} \!\left( t_n \inf_{u \in [-B_1, B_1]} \tilde{\pi}_{SH}(u ) \right)^{S_{1n}},
    \end{align*}
    where $t_n = {\epsilon_n}/[{2 (D_{1n}+1)^{L_{1n}} L_{1n} (B_{1} \vee 1)^{L_{1n}-1}}]$. Therefore,
        \begin{align}
        \label{eqn:shrprcon}
\begin{split}
                -\log \Pi(A_{\epsilon_n}) &\lesssim -S_{1n} \log \!\left( t_n \inf_{u \in [-B_1, B_1]} \tilde{\pi}_{SH}(u ) \right) \\
                &\quad- (T_{1n} - S_{1n}) \log (1-u_n) +\log n\\
                &\lesssim S_{1n} (\log n)^2 - T_{1n} \log(1-S_{1n}/T_{1n}) + \log n \\
                &= S_{1n} (\log n)^2 + T_{1n}\left( S_{1n}/T_{1n} + o\left(S_{1n} /T_{1n} \right) \right) + \log n \\
                &\lesssim S_{1n} (\log n)^2 + S_{1n} + o\left(S_{1n}\right) + \log n \\
                &\lesssim n\epsilon_n^2.
        \end{split}
        \end{align}
    For the second inequality, we used the fact that
    \begin{align*}
        -\log (1- u_n) \leq - \log\!\left(1 - {e^{-C_A(\log n)^2}}\right) \leq - \log\left(1 - S_{1n}/T_{1n}\right)
    \end{align*}
    for sufficiently large $n$, as implied by Assumption~\aref{cond:shr_spike}.
\end{proof}

\subsection{Proof of \Cref{thm:adaptive_estimation}}\label{subsec:proof_adaptive}

\begin{proof}
Let $\mathcal F={\Phi}(\tilde L_{n})$.
By \cref{lem:consistency_unknown_var}, it suffices to show that there exists a subset $\mathcal{F}_n \subset \mathcal{F}$
such that
    \begin{enumerate}[leftmargin=1.8em,itemsep=-0.3em,topsep=0em,label=(\alph*)]
        \item $\log \Covering{\epsilon_n, \mathcal{F}_n, \norm{\cdot}_{L^2(P_X)}} \lesssim n\epsilon_n^2$
        \item $-\log \Pi(A_{\epsilon_n}) \lesssim n \epsilon_n^2$
        \item $\sup_{\sigma \in [\underline{\sigma}, \overline{\sigma}]} \Pi(\mathcal{F} \setminus \mathcal{F}_n \mid \sigma) = o\!\left(e^{-C^\dprime n \epsilon_n^2}\right)$
    \end{enumerate}
    for a sufficiently large $C^\dprime$.
    We first prove (i), the case corresponding to the spike-and-slab prior. 
    Define $\tilde{D}_n = C_D N_n$ and $\tilde{S}_n = C_S N_n \log n$ for sufficiently large $C_D, C_S > 0$. 
Define the sieve
    \begin{equation*}
        \mathcal{F}_n = \bigcup_{D\le \tilde{D}_n} \bigcup_{S\le \tilde{S}_n} {\Phi}(\tilde{L}_{n},D,S,B_n),
    \end{equation*}
        where $B_n = n$.
        To verify the entropy bound, apply \Cref{lem:covering} to obtain
        \begin{align}
        \label{eqn:adaentropy}
        \begin{split}
       \Covering{\epsilon_n, \mathcal{F}_n, \norm{\cdot}_{\infty} }
            &\leq \sum_{D\le \tilde{D}_n}\sum_{S\le \tilde{S}_n} \bigg( \frac{2}{\epsilon_n} \tilde{L}_{n} (B_n \vee 1 )^{\tilde{L}_{n}}  (D+1)^{2\tilde{L}_{n}} \bigg)^{S+1} \\
            &\leq \tilde{D}_n \tilde{S}_n \bigg( \frac{2}{\epsilon_n} \tilde{L}_{n}(B_n \vee 1 )^{\tilde{L}_{n}}
                (\tilde{D}_n+1)^{2\tilde{L}_{n}} \bigg)^{\tilde{S}_n+1}.
        \end{split}
        \end{align}
Hence, (a) holds as in the proof of \Cref{thm:reg_ss}.
    To verify (b), observe that there exists a constant $C^{\prime\prime\prime} > 0$ such that 
    \begin{align*}   \min\{\pi_D(\tilde{D}_n),\pi_S(\tilde{S}_n)\} \gtrsim \Exp{-C^{\prime\prime\prime} n \epsilon_n^2}.
    \end{align*}
     Since $D_{1n} \leq \tilde{D}_n$ and $S_{1n} \leq \tilde{S}_n$,
    by \cref{lem:anisotropic_approx} and \cref{rmk:hyper_bigo}, there exists  $\hat{f}_n = f_{\hat{\theta}} \in \Phi(\tilde{L}_{n}, \tilde{D}_n,\tilde{S}_n, B_n)$ satisfying \eqref{eqn:mle_l2}.
Following the proof of \cref{thm:reg_ss}, it can be shown that
     \begin{align*}
        -\log \Pi\!\left(f \in \mathcal{F}_n: \norm{f-f_0}_{L^2(P_X)} \leq \epsilon_n/2 \mid \tilde{D}_n,\tilde{S}_n\right) \lesssim n\epsilon_n^2.
    \end{align*}
    Therefore,
        \begin{align}
        \label{eqn:adaprcon}
        \begin{split}
        -\log \Pi(A_{\epsilon_n})
        &\lesssim -\log\Pi\!\left(f \in \mathcal{F}_n: \norm{f - f_0}_{L^2(P_X)} \leq \epsilon_n/2\right) - \log n \\
        &\leq - \log \pi_D(\tilde{D}_n) - \log \pi_S(\tilde{S}_n) \\
        &\quad - \log \Pi\!\left(f \in \mathcal{F}_n: \norm{f-f_0}_{L^2(P_X)} \leq \epsilon_n / 2 \mid \tilde{D}_n,\tilde{S}_n\right) - \log n \\
        &\lesssim n\epsilon_n^2,        \end{split}
        \end{align}
        which verifies (b).
        Now, observe that
        \begin{align*}
\Pi(\mathcal{F} \setminus \mathcal{F}_n \mid \sigma) 
&= \Pi(\mathcal{F} \setminus \mathcal{F}_n ) \\
&\leq \Pi(D>\tilde{D}_n)+\Pi(S > \tilde{S}_n) \\
&\quad+\sum_{D\le \tilde{D}_n}\sum_{S\le \tilde{S}_n}\pi_D(D)\pi_S(S)\Pi \!\left( \exists j: \abs{\theta_j} > B_{n} \mid \tilde L_n,D, S\right).
        \end{align*}
        First, we bound the tail probability of $D$ as
        \begin{align*}
            \Pi(D > \tilde D_n) 
            &\lesssim \sum_{D>\tilde D_n} e^{-\lambda_D D (\log D)^3} 
            \leq \sum_{D>\tilde D_n} e^{-\lambda_D  D (\log \tilde D_n)^3} 
        \lesssim {e^{-\lambda_D \tilde D_n (\log \tilde D_n)^3}}\le e^{-C_1n\epsilon_n^2},
        \end{align*}
        for a sufficiently large $C_1>0$, provided that $C_D$ is large enough.
        Similarly, we obtain $\Pi(S > \tilde{S}_n) \lesssim e^{-C_2 n \epsilon_n^2}$ for a sufficiently large $C_2>0$, provided that $C_S$ is sufficiently large.
Lastly, following the argument used in the proof of \Cref{thm:reg_ss},
\begin{align*}
&\sum_{D\le \tilde{D}_n}\sum_{S\le \tilde{S}_n}\pi_D(D)\pi_S(S)\Pi \!\left( \exists j: \abs{\theta_j} > B_{n} \mid \tilde L_n,D, S\right) \\
&= \sum_{D\le \tilde{D}_n}\sum_{S\le \tilde{S}_n}\pi_D(D)\pi_S(S)[1-(1-v_n)^S]\\
&
\le 1-(1-v_n)^{\tilde S_n}
\\
&=o\!\left(e^{-C_3 n\epsilon_n^2}\right),
\end{align*}
for a sufficiently large $C_3>0$.
Therefore, (c) holds.

To prove (ii), define the sieve
    \begin{equation*}
        \mathcal{F}_n = \bigcup_{D\le \tilde{D}_n} {\Phi}(\tilde L_{1n}, D, S_{1n}, B_{n}, a_n).
    \end{equation*}
Using \cref{lem:covering}, (a) is verified as in \eqref{eqn:adaentropy}.
Moreover, following \eqref{eqn:shrprcon} and \eqref{eqn:adaprcon}, we obtain
\begin{align*}
            -\log \Pi(A_{\epsilon_n})
        &\leq - \log \pi_D(\tilde{D}_n)
        - \log \Pi\!\left(f \in \mathcal{F}_n: \norm{f-f_0}_{L^2(P_X)} \leq \epsilon_n / 2 \mid \tilde{L}_n, \tilde{D}_n \right) - \log n \\
        &\lesssim n\epsilon_n^2,    
\end{align*}
which verifies (b). 
Lastly, we obtain
       \begin{align*}
&\Pi(\mathcal{F} \setminus \mathcal{F}_n \mid \sigma) \\
&= \Pi(\mathcal{F} \setminus \mathcal{F}_n ) \\
&\leq \Pi(D>\tilde{D}_n) \\
&\quad+\sum_{D\le \tilde{D}_n}\pi_D(D)\!\left[\Pi\! \left( \exists j: \abs{\theta_j} > B_{n} \mid  \tilde L_{1n},D \right)  + \Pi\!\left(\sum_{j=1}^{T_{1n}} I(\abs{\theta_j} >a_n) > S_{1n}\mid  \tilde L_{1n},D \right)\right].
        \end{align*}
        Following the calculation in the proof of \Cref{thm:shrinkage},
        we verify (c).
\end{proof}

\subsection{Proof of \cref{thm:general_function}}\label{subsec:proof_general}

\begin{proof}
We only prove the part corresponding to \Cref{thm:reg_ss} under Assumption~\aref{assum:a-4}. Other arguments proceed similarly. Let $\mathcal{F} = {\Phi}(L_{2n},D_{2n}, S_{2n})$ and $\mathcal{F}_n = {\Phi}(L_{2n},D_{2n},S_{2n},B_2)$. We verify conditions (a) and (b) of \Cref{lem:consistency_unknown_var}. 
Using \cref{lem:covering},
\begin{equation*}
    \begin{aligned}
        \log \Covering{\epsilon_n, \mathcal{F}_n, \norm{\cdot}_{L^2(P_X)} }
        &\leq \log \Covering{\epsilon_n, \mathcal{F}_n, \norm{\cdot}_{\infty}} \\
        &\leq (S_{2n}+1) \bigg[ \log L_{2n} + L_{2n} \Log{ (B_{2} \vee 1) (D_{2n} + 1)^2}  - \log \frac{\epsilon_n}{2} \bigg] \\
        &\lesssim N_n (\log n)^3 \\
        &\lesssim n\epsilon_n^2
    \end{aligned}
\end{equation*}
for sufficiently large $n$. The last inequalities hold because
\begin{equation*}
    L_{2n} \asymp \log n,\quad D_{2n} \asymp N_n,\quad S_{2n} \asymp N_n \log n
\end{equation*}
for fixed $d^{(h)}$. Thus, (a) holds. 
Using \cref{lem:composite_anisotropic_approx} and Assumption \aref{assum:a-1}, there exists $\hat{f}_n = f_{\hat{\theta}} \in \mathcal{F}_n$ such that
\begin{equation*}
    \|\hat{f}_n - f_0\|_{L^2(P_X)} \leq C_1  N_n^{-\tilde{\bfs}^\ast} \leq \epsilon_n/4
\end{equation*}
for some $C_1>0$ and all sufficiently large $n$. 
Let $\hat{\gamma}$ denote the index set of nonzero components in $\hat{\theta}$, and let $\hat{\theta}_{\hat{\gamma}}$ denote the corresponding nonzero values. We define $\Theta(\hat{\gamma}; L_{2n}, D_{2n}, S_{2n}, B_2) \subset \Theta(L_{2n}, D_{2n}, S_{2n}, B_2)$ as the subset of parameter space in which only the components indexed by $\hat{\gamma}$ are nonzero. The corresponding NN space is denoted by
\begin{equation*}
    \mathcal{F}_n(\hat{\gamma}) = {\Phi}(\hat{\gamma}; L_{2n}, D_{2n}, S_{2n}, B_2) .
\end{equation*}
Using Assumption \aref{assum:a-3}, we have
\begin{equation*}
    \log \Pi(\abs{\sigma - \sigma_0} \leq  \epsilon_n/2) \gtrsim \log \epsilon_n \gtrsim -\log n
\end{equation*}
and hence
\begin{equation*}
    \Pi(A_{\epsilon_n}) \gtrsim \Pi(f \in \mathcal{F}_n: \norm{f - f_0}_{L^2(P_X)} \leq \epsilon_n/2) - \log n.
\end{equation*}
Using the earlier result,
\begin{equation*}
    \begin{aligned}
        \Pi\!\left(f \in \mathcal{F}_n: \norm{f - f_0}_{L^2(P_X)} \leq \epsilon_n/2\right)\ge \Pi\!\left(f \in \mathcal{F}_n(\hat{\gamma}): \|f - \hat{f}_n\|_{\infty} \leq \epsilon_n/4\right).
    \end{aligned}
\end{equation*}
Using \cref{lem:nn_norm},
    \begin{align*}
        &\Pi\bigg( f \in \mathcal{F}_n(\hat{\gamma}): \|f - \hat{f}_n\|_{\infty} \leq \epsilon_n/4\bigg) \\
        &\ge \Pi\bigg(\theta \in \mathbb{R}^{T_{2n}}: \theta_{\hat{\gamma}^c} = 0, \norm{\theta_{\hat{\gamma}}}_{\infty} \leq B_2, 
        \|\hat{\theta}_{\hat\gamma} - \theta_{\hat\gamma}\|_{\infty} \leq \frac{\epsilon_n}{4 (D_{2n}+1)^{L_{2n}} L_{2n} (B_2 \vee 1)^{L_{2n}-1}} \bigg) \\
        &\ge \left( t_{2n} \inf_{u \in [-B_2, B_2]} \tilde{\pi}_{SL}(u) \right)^{S_{2n}} \binom{T_{2n}}{S_{2n}}^{-1} \\
        &\ge \left( t_{2n} (D_{2n} + 1)^{-L_{2n}} \inf_{u \in [-B_2, B_2]} \tilde{\pi}_{SL}(u) \right)^{S_{2n}},
    \end{align*}
where $t_{2n} = \epsilon_n/[{2(D_{2n}+1)^{L_{2n}} L_{2n} (B_2 \vee 1)^{L_{2n}-1}}]$.
Therefore,
\begin{equation*}
    \begin{aligned}
        -\log \Pi(A_{\epsilon_n}) 
        &\lesssim -S_{2n} \log \!\left( t_{2n} (D_{2n} + 1)^{-L_{2n}} \inf_{u \in [-B_2, B_2]} \tilde{\pi}_{SL}(u) \right) + \log n\\
        &\leq S_{2n} (\log n)^2 \\
        &\lesssim n\epsilon_n^2,
    \end{aligned}
\end{equation*}
which verifies (b).
\end{proof}

\end{document}